\documentclass[twoside]{article}

\usepackage[accepted]{aistats2021}

\usepackage[utf8]{inputenc} 
\usepackage[T1]{fontenc}    

\usepackage[colorlinks]{hyperref}       
\hypersetup{
    colorlinks = true,
    linkcolor = blue,
    anchorcolor = blue,
    citecolor = blue,
    filecolor = blue,
    urlcolor = blue
}
     
\usepackage{url}            
\usepackage{booktabs}       
\usepackage{amsfonts}       
\usepackage{amsmath}
\usepackage{nicefrac}       
\usepackage{xcolor}
\usepackage{algorithm}
\usepackage[noend]{algpseudocode}

\algrenewcommand\algorithmicforall{\textbf{foreach}}
\usepackage{textcomp}
\usepackage{subcaption,graphicx}
\usepackage[font=footnotesize,labelfont=bf]{caption}

\usepackage{array,multirow}
\usepackage[nomessages]{fp}

\newcommand{\mc}{\mathcal}
\newcommand{\mb}{\boldsymbol}
\newcommand{\mbb}{\mathbb}

\DeclareMathOperator*{\tr}{tr}

\usepackage{xstring}

\newcommand{\bestresult}[2]{\FPeval{\m}{round(#1, 2)}\FPeval{\s}{round(#2, 2)} \StrGobbleLeft{\s}{1}[\s]${\scriptstyle \mb{\m_{{\pm \s}}}}$}
\newcommand{\result}[2]{\FPeval{\m}{round(#1, 2)}\FPeval{\s}{round(#2, 2)} \StrGobbleLeft{\s}{1}[\s]${\scriptstyle \m_{\scriptstyle \pm \s}}$}

\usepackage{mathtools}

\DeclarePairedDelimiter{\floor}{\lfloor}{\rfloor}

\usepackage{amsthm}
\newtheorem{proposition}{Proposition}[section]

\newcommand{\tbf}[1]{\textbf{#1}}
\newcommand{\figref}[1]{Fig.~\ref{#1}}
\newcommand{\tabref}[1]{Table~\ref{#1}}

\usepackage{relsize}
\usepackage{tikz}
\usetikzlibrary{positioning,fit,calc,decorations.pathreplacing,arrows,backgrounds,shapes}
\usetikzlibrary{decorations.pathreplacing,angles,quotes}

\usetikzlibrary{automata,arrows}

\newcommand{\nonl}{\renewcommand{\nl}{\let\nl\oldnl}}

\usepackage[]{natbib}
\begin{document}

\twocolumn[
\aistatstitle{Graphical Normalizing Flows}

\aistatsauthor{ Antoine Wehenkel \And Gilles Louppe }

\aistatsaddress{ ULiège \And ULiège} ]

\begin{abstract}
Normalizing flows model complex probability distributions by combining a base distribution with a series of bijective neural networks. State-of-the-art architectures rely on coupling and autoregressive transformations to lift up invertible functions from scalars to vectors. 
In this work, we revisit these transformations as probabilistic graphical models, showing they reduce to Bayesian networks with a pre-defined topology and a learnable density at each node.
From this new perspective, we propose the graphical normalizing flow, a new invertible transformation with either a prescribed or a learnable graphical structure. 
This model provides a promising way to inject domain knowledge into normalizing flows while preserving both the interpretability of Bayesian networks and the representation capacity of normalizing flows. We show that graphical conditioners discover relevant graph structure when we cannot hypothesize it. In addition, we analyze the effect of $\ell_1$-penalization on the recovered structure and on the quality of the resulting density estimation. Finally, we show that graphical conditioners lead to competitive white box density estimators.
Our implementation is available at \url{https://github.com/AWehenkel/DAG-NF}.
\end{abstract}

\section{Introduction}
Normalizing flows~\citep[NFs, ][]{NF, tabak2010density, tabak2013family, rippel2013high} have proven to be an effective way to model complex data distributions with neural networks. These models map data points to latent variables through an invertible function while keeping track of the change of density caused by the transformation. In contrast to variational auto-encoders (VAEs) and generative adversarial networks (GANs), NFs provide access to the exact likelihood of the model's parameters, hence offering a sound and direct way to optimize the network parameters. Normalizing flows have proven to be of practical interest as demonstrated by \citet{parallel_wavenet, kim2018flowavenet} and \citet{ prenger2019waveglow} for speech synthesis, by \citet{NF, kingma2016improved} and \citet{berg2018sylvester} for variational inference or by \citet{SNL} and \citet{greenberg2019automatic} for simulation-based inference. Yet, their usage as a base component of the machine learning toolbox is still limited in comparison to GANs or VAEs. Recent efforts have been made by \citet{flows-review} and \citet{flow_review_2} to define the fundamental principles of flow design and by \citet{neural-spline-flows} to provide coding tools for modular implementations. We argue that normalizing flows would gain in popularity by offering stronger inductive bias as well as more interpretability. 

Sometimes forgotten in favor of more data oriented methods, probabilistic graphical models (PGMs) have been popular for modeling complex data distributions while being relatively simple to build and read \citep{PGM-book, johnson2016composing}. Among PGMs, Bayesian networks~\citep[BNs, ][]{Pearl-BN} offer an appealing balance between modeling capacity and simplicity. Most notably, these models have been at the basis of expert systems before the big data era (e.g. \citep{BN-app-1, BN-app-2, BN-app-3}) and were commonly used to merge qualitative expert knowledge and quantitative information together. On the one hand, experts stated independence assumptions that should be encoded by the structure of the network. On the other hand, data were used to estimate the parameters of the conditional probabilities/densities encoding the quantitative aspects of the data distribution. 
These models have progressively received less attention from the machine learning community in favor of other methods that scale better with the dimensionality of the data. 

Driven by the objective of integrating intuition into normalizing flows and the proven relevance of BNs for combining qualitative and quantitative reasoning, we summarize our contributions as follows:
\textbf{(i)} From the insight that coupling and autoregressive transformations can be reduced to Bayesian networks with a fixed topology, we introduce the more general graphical conditioner for normalizing flows, featuring either a prescribed or a learnable BN topology; \textbf{(ii)} We show that using a correct prescribed topology leads to improvements in the modeled density compared to autoregressive methods. When the topology is not known we observe that, with the right amount of $\ell_1$-penalization, graphical conditioners discover relevant relationships; \textbf{(iii)} In addition, we show that graphical normalizing flows perform well in a large variety of density estimation tasks compared to classical black-box flow architectures. 
\section{Background}

\paragraph{Bayesian networks}
A Bayesian network is a directed acyclic graph (DAG) that represents independence assumptions between the components of a random vector. Formally, let $\mb{x} = \left[x_1, \hdots, x_d\right]^T \in \mathbb{R}^d$ be a random vector distributed under $p_{\mb{x}}$. A BN associated to $\mb{x}$ is a directed acyclic graph made of $d$ vertices representing the components $x_i$ of $\mb{x}$. In this kind of network, the absence of edges models conditional independence between groups of components through the concept of d-separation~\citep{d-separation}. A BN is a valid representation of a random vector $\mb{x}$ iff its density can be factorized as 
\begin{align}
    p_{\mb{x}}(\mb{x}) = \prod^d_{i=1}p(x_i|\mathcal{P}_i),\label{eq:BN-fact}
\end{align} 
where  $\mathcal{P}_i = \{j: A_{i,j} = 1 \}$ denotes the set of parents of the vertex $i$ and $A \in \{0, 1\}^{d\times d}$ is the adjacency matrix of the BN. As an example, \figref{fig:mono-step-flows-a} is a valid BN for any distribution over $\mb{x}$ because it does not state any independence and leads to a factorization that corresponds to the chain rule. However, in practice we seek for a sparse and valid BN which models most of the independence between the components of $\mb{x}$, leading to an efficient factorization of the modeled probability distribution. It is worth noting that making hypotheses on the graph structure is equivalent to assuming certain conditional independence between some of the vector's components.


\begin{figure}
    \centering
    \begin{subfigure}{.29\textwidth}
        \centering
        \begin{tikzpicture}[
          node distance=.7cm and 1.cm,
          var_x/.style={draw, circle, text width=.4cm, align=center}
        ]
            \node[var_x] (x1) {$x_1$};
            \node[var_x, right=of x1] (x2) {$x_2$};
            \node[var_x, below=of x1] (x3) {$x_3$};
            \node[var_x, right=of x3] (x4) {$x_4$};
            \path (x1) edge[-latex] (x2);
            \path (x1) edge[-latex] (x3);
            \path (x1) edge[-latex] (x4);
            \path (x2) edge[-latex] (x3);
            \path (x2) edge[-latex] (x4);
            \path (x3) edge[-latex] (x4);
        \end{tikzpicture}
        \caption{} \label{fig:mono-step-flows-a}
    \end{subfigure}~\hspace{-4.8em}
    \begin{subfigure}{.3\textwidth}
    \centering
        \begin{tikzpicture}[
          node distance=.7cm and 1.cm,
          var_x/.style={draw, circle, text width=.4cm, align=center}
        ]
            \node[var_x] (x1) {$x_1$};
            \node[var_x, right=of x1] (x2) {$x_2$};
            \node[var_x, below=of x1] (x3) {$x_3$};
            \node[var_x, right=of x3] (x4) {$x_4$};
            \path (x1) edge[-latex] (x3);
            \path (x1) edge[-latex] (x4);
            \path (x2) edge[-latex] (x3);
            \path (x2) edge[-latex] (x4);
        \end{tikzpicture}
        \caption{} \label{fig:mono-step-flows-b}
    \end{subfigure}
\\
    \begin{subfigure}{.5\textwidth}
    \centering
        \begin{tikzpicture}[
            node distance=.7cm and 1.cm,
            var_z/.style={draw,circle,text width=.4cm,align=center},
            var_x/.style={draw, circle, double, text width=.35cm, align=center}
        ]
            \node[var_x] (x1) {$x_1$};
            \node[var_x, right=of x1] (x2) {$x_2$};
            \node[var_x, below=of x1] (x3) {$x_3$};
            \node[var_x, right=of x3] (x4) {$x_4$};
            \node[var_z, left=of x1] (z1) {$z_1$};
            \node[var_z, right=of x2] (z2) {$z_2$};
            \node[var_z, left=of x3] (z3) {$z_3$};
            \node[var_z, right=of x4] (z4) {$z_4$};
            \path (x1) edge[-latex] (x3);
            \path (x1) edge[-latex] (x4);
            \path (x2) edge[-latex] (x3);
            \path (x2) edge[-latex] (x4);
            \foreach \x in {1,...,4}
                \path (z\x) edge[-latex] (x\x);
        \end{tikzpicture}
        \caption{} \label{fig:mono-step-flows-c}
    \end{subfigure}
    \caption{Bayesian networks equivalent to normalizing flows made of a single transformation step. (\tbf{a}) Autoregressive conditioner. (\tbf{b}) Coupling conditioner. (\tbf{c}) Coupling conditioner, with latent variables  shown explicitly. Double circles stand for deterministic functions of the parents.} \label{fig:mono-step-flows}
\end{figure}
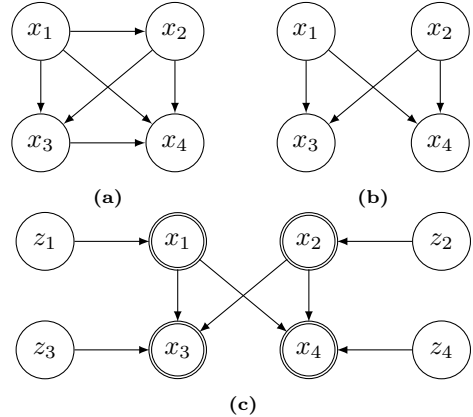

\paragraph{Normalizing flows}
A normalizing flow is defined as a sequence of invertible transformation steps $\mb{g}_k : \mathbb{R}^d \to \mathbb{R}^d$  ($k=1, ..., K$)  composed together to create an expressive invertible mapping $\mb{g} := \mb{g}_1 \circ \dots \circ \mb{g}_K : \mathbb{R}^d \to \mathbb{R}^d$. 
This mapping can be used to perform density estimation, using $\mb{g}(\cdot ;\mb{\theta}): \mbb{R}^d \rightarrow \mbb{R}^d$ to map a sample $\mb{x} \in \mathbb{R}^d$ onto a latent vector $\mb{z} \in \mbb{R}^d$ equipped with a prescribed density $p_{\mb{z}}(\mb{z})$ such as an isotropic Normal.
The transformation $\mb{g}$ implicitly defines a density $p(\mb{x}; \mb{\theta})$ as given by the change of variables formula,
\begin{equation}
    p(\mb{x}; \mb{\theta}) = p_{\mb{z}}(\mb{g}(\mb{x};\mb{\theta})) \left| \det  J_{\mb{g}(\mb{x};\mb{\theta})} \right|, \label{eq:NF_DE}
\end{equation}
where $J_{\mb{g}(\mb{x};\mb{\theta})}$ is the Jacobian of $\mb{g}(\mb{x};\mb{\theta})$ with respect to $\mb x$. The resulting model is trained by maximizing the likelihood of the model's parameters $\mb{\theta}$ given the training dataset $\mb{X} = \{\mb{x}^1, ..., \mb{x}^N\}$.
Unless needed, we will not distinguish between $\mb{g}$ and $\mb{g}_k$ for the rest of our discussion.

In general the steps $\mb{g}$ can take any form as long as they define a bijective map. Here, we focus  on a sub-class of normalizing flows for which the steps can be mathematically described as 
\begin{align}
    \mb{g}(\mb{x}) = \begin{bmatrix}
g^1(x_{1}; \mb{c}^1(\mb{x}))  & \hdots & g^d(x_{d}; \mb{c}^d(\mb{x}))
\end{bmatrix}^T,\label{eq:gnf}
\end{align}
where the $\mb{c}^i$ are the \tbf{conditioners} which role is to constrain the structure of the Jacobian of $\mb{g}$. 
The functions $g^i$, partially parameterized by their conditioner, must be invertible with respect to their input variable $x_i$. They are often referred to as transformers, however in this work we will use the term \tbf{normalizers} to avoid any confusion with attention-based transformer architectures. 


The conditioners examined in this work can be combined with any normalizer. In particular, we consider affine and monotonic normalizers. An affine normalizer $g: \mathbb{R} \times \mathbb{R}^2 \rightarrow \mathbb{R}$ can be expressed as
$g(x;m, s) = x\exp(s) + m$, where $m \in \mathbb{R}$ and $s \in \mathbb{R}$ are computed by the conditioner. There exist multiple methods to parameterize monotonic normalizers \citep{huang2018neural, de2020block, neural-spline-flows, sos}, but in this work we rely on Unconstrained Monotonic Neural Networks~\citep[UMNNs, ][]{UMNN} which can be expressed as
$g(x; \mb{c}) = \int^x_0 f(t, \mb{c}) dt + \beta(\mb{c})$, where $\mb{c} \in \mathbb{R}^{|\mb{c}|}$ is an embedding made by the conditioner and $f: \mathbb{R}^{|\mb{c}| + 1} \rightarrow \mathbb{R}^+$ and  $\beta: \mathbb{R}^d\rightarrow\mathbb{R}$ are two neural networks respectively with a strictly positive scalar output and a real scalar output. \cite{huang2018neural} proved NFs built with autoregressive conditioners and monotonic normalizers are universal density approximators of continuous random variables.

\section{Normalizing flows as Bayesian networks} \label{sec:NF-BN}


\paragraph{Autoregressive conditioners}
Due to computing speed considerations, NFs are usually composed of transformations for which the determinant of the Jacobian can be computed efficiently, as otherwise its evaluation would scale cubically with the input dimension. 
A common solution is to use autoregressive conditioners, i.e., such that 
$$\mb{c}^i(\mb{x}) = \mb{h}^i\left(\begin{bmatrix} x_1 & \hdots & x_{i-1} \end{bmatrix}^T\right)$$ are functions $\mb{h}^i$ of the first $i-1$ components of $\mb{x}$. This particular form constrains the Jacobian of $\mb{g}$ to be lower triangular, which makes the computation of its determinant $\mc{O}(d)$. 

For the multivariate density $p(\mb{x}; \mb{\theta})$ induced by $\mb{g}(\mb{x};\mb{\theta})$ and $p_{\mb{z}}(\mb{z})$, we can use the chain rule to express the joint probability of $\mb{x}$ as a product of $d$ univariate conditional densities, 
\begin{align}
    p(\mb{x}; \mb{\theta}) = p(x_1; \mb{\theta})\prod^{d}_{i=2}p(x_{i}|\mb{x}_{1:i-1}; \mb{\theta}). \label{eq:AF-fact}
\end{align}
When $p_{\mb{z}}(\mb{z})$ is a factored distribution $p_{\mb{z}}(\mb{z}) = \prod^{d}_{i=1}p(z_i)$, we identify that each component $z_i$ coupled with the corresponding functions $g^i$ and embedding vectors $\mb{c}^i$ encode for the conditional $p(x_{i}|\mb{x}_{1:i-1}; \mb{\theta})$. Therefore, and as illustrated in \figref{fig:mono-step-flows-a}, autoregressive transformations can be seen as a way to model the conditional factors of a BN that does not state any independence but relies on a predefined node ordering. This becomes clear if we define $\mathcal{P}_i = \{x_1, \hdots, x_{i-1}\}$ and compare \eqref{eq:AF-fact} with \eqref{eq:BN-fact}.

The complexity of the conditional factors  strongly depends on the ordering of the vector components. While not hurting the universal representation capacity of normalizing flows, the arbitrary ordering used in autoregressive transformations leads to poor inductive bias and to factors that are most of the time difficult to learn. 
In practice, one often alleviates the arbitrariness of the ordering by stacking multiple autoregressive transformations combined with random permutations on top of each other. 

\paragraph{Coupling conditioners}
Coupling layers~\citep{NICE} are another popular type of conditioners that lead to a bipartite structure. The conditioners $\mb{c}^i$ made from coupling layers are defined as
\begin{align*}
    \mb{c}^i(\mb{x}) = 
    \begin{cases} 
    \underline{\mb{h}}^i \quad \text{if} \quad i \leq k\\
    \mb{h}^i\left(\begin{bmatrix} x_1 & \hdots & x_k \end{bmatrix}^T\right) \quad \text{if} \quad i > k\\
    \end{cases}
\end{align*}
where the underlined $\underline{\mb{h}}^i \in \mathbb{R}^{|\mb{c}|}$ denote constant values and $k \in \{1, \hdots, d\}$ is a hyper-parameter usually set to $\floor*{\frac{d}{2}}$. 
As for autoregressive conditioners, the Jacobian of $\mb{g}$ made of coupling layers is lower triangular.  
Again, and as shown in \figref{fig:mono-step-flows-b} and \ref{fig:mono-step-flows-c}, these transformations can be seen as a specific class of BN where $\mathcal{P}_i=\{\}$ for $i \leq k$ and $\mathcal{P}_i = \{1, ..., k\}$ for $i > k$. 
D-separation can be used to read off the independencies stated by this class of BNs such as the conditional independence between each pair in $\mathbf{x}_{k+1:d}$ knowing $\mathbf{x}_{1:k}$. 
For this reason, and in contrast to autoregressive transformations, coupling layers are not by themselves universal density approximators even when associated with very expressive normalizers $g^i$ \citep{wehenkel2020you}.
In practice, these bipartite structural independencies can be relaxed by stacking multiple layers, and may even recover an autoregressive structure. They also lead to useful inductive bias, such as in the multi-scale architecture with checkerboard masking~\citep{RealNVP, GLOW}.

\begin{figure}[h]
\begin{algorithm}[H]
  \caption{Sampling} \label{alg:inversion}
  \small
  \begin{algorithmic}[1]
    \State $\mb{z} \sim \mathcal{N}(\mb{0}, I)$
    \State $\mb{x} \gets \mb{z}$
    \Repeat
      \State $\mb{c}^i \gets \mb{h}^i(\mb{x} \odot A_{i, :}) \quad \forall i \in \{1, \dots, d\}$
      \State $x_i \gets (g^i)^{-1}(z_i; \mb{c}^i, \mb{\theta}) \quad \forall i \in \{1, \dots, d\}$
    \Until{$\mb{x}$ converged}
  \end{algorithmic}
\end{algorithm}
\end{figure}
\section{Graphical normalizing flow}
\subsection{Graphical conditioners}
Following up on the previous discussion, we introduce the graphical conditioner architecture. 
We motivate our approach by observing that the topological ordering (a.k.a. ancestral ordering) of any BN leads to a lower triangular adjacency matrix whose determinant is equal to the product of its diagonal terms (proof in Appendix \ref{app:proof-jac-gnf}). 
Therefore, conditioning factors $\mb{c}^i(\mb{x})$ selected by following a BN adjacency matrix necessarily lead to a transformation $\mb{g}$ whose Jacobian determinant remains efficient to compute. 

Formally, given a BN with adjacency matrix $A \in \{0, 1\}^{d\times d}$, we define the \tbf{graphical conditioner} as being
\begin{align}
    \mb{c}^i(\mb{x}) = \mb{h}^i(\mb{x} \odot A_{i,:}),
\label{eq:gnf}
\end{align}
where $\mb{x} \odot A_{i,:}$ is the element-wise product between the vector $\mb{x}$ and the $i^{\text{th}}$ row of $A$ -- i.e., the binary vector $A_{i, :}$ is used to mask on $\mb{x}$.
NFs built with this new conditioner architecture can be inverted by sequentially inverting each component in the topological ordering.
In our implementation the neural networks modeling the $h^i$ functions are shared to save memory and they take an additional input that one-hot encodes the value $i$. An alternative approach would be to use a masking scheme similar to what is done by \citet{MADE} in MADE as suggested by \cite{DAG-3}.

The graphical conditioner architecture can be used to learn the conditional factors in a continuous BN while elegantly setting structural independencies prescribed from domain knowledge. In addition, the inverse of NFs built with graphical conditioners is a simple as it is for autoregressive and coupling conditioners, Algorithm \ref{alg:inversion} describes an inversion procedure.
We also now note how these two conditioners are just special cases in which the adjacency matrix reflects the classes of BNs discussed in Section~\ref{sec:NF-BN}. 

\subsection{Learning the topology}\label{sec:learn_topo}
In many cases, defining the whole structure of a BN is not possible due to a lack of knowledge about the problem at hand.
Fortunately, not only is the density at each node learnable, but also the DAG structure itself:
defining an arbitrary topology and ordering, as it is implicitly the case for autoregressive and coupling conditioners, is not necessary.

Building upon \textit{Non-combinatorial Optimization via Trace Exponential and Augmented lagRangian for
Structure learning}~\citep[NO TEARS, ][]{DAG-1}, we convert the combinatorial optimization of score-based learning of a DAG into a continuous optimization by relaxing the domain of $A$ to real numbers instead of binary values. That is,
\begin{align}\label{eq:DAG-relax}
\begin{aligned}
\max_{A\in\mathbb{R}^{d\times d}} & \ \ F(A) \\
\text{s.t.} & \ \ \mathcal{G}(A) \in \mathsf{DAGs}
\end{aligned}
\quad \iff \quad 
\begin{aligned}
\max_{A\in\mathbb{R}^{d\times d}} & \ \  F(A) \\
\text{s.t.} & \ \ w(A) = 0,
\end{aligned}
\end{align}
where $\mathcal{G}(A)$ is the graph induced by the weighted adjacency matrix $A$ and $F: \mathbb{R}^{d \times d}\rightarrow \mathbb{R}$ is the log-likelihood of the graphical NF $\mb{g}$ plus a regularization term, i.e., 
\begin{align}
     F(A) = \sum^N_{j=1}\log\big(p(\mb{x}^j; \mb{\theta})\big) + \lambda_{\ell_1} ||A||_1, \label{eq:score}
\end{align}

where $\lambda_{\ell_1}$ is an $\ell_1$-regularization coefficient and $N$ is the number of training samples $\mb{x}^i$. The likelihood is computed as
\begin{align}
    p(\mb{x}; \mb{\theta})=& p_{\mb{z}}(\mb{g}(\mb{x}; \mb{\theta})) \prod_{i=1}^{d} \left| \frac{\partial g^i(x_i;\mb{h}^i(\mb{x} \odot A_{i,:}), \mb{\theta})}{\partial x_i} \right|. \nonumber
\end{align}
The function $w(A)$ that enforces the acyclicity is expressed as suggested by \citet{DAG-2} as
$$w(A) := \tr\left((I + \alpha A)^d\right) - d \propto	\tr\left(\sum^d_{k=1} \alpha^k A^k\right),$$
where $\alpha \in \mathbb{R}_+$ is a hyper-parameter that avoids exploding values for $w(A)$.
In the case of positively valued $A$, an element $(i, j)$ of $A^k = \underbrace{A A...A}_{\text{k terms}}$ is non-null if and only if there exists a path going from node $j$ to node $i$ that is made of exactly $k$ edges. Intuitively $w(A)$ expresses to which extent the graph is cyclic. Indeed, the diagonal elements $(i, i)$ of $A^k$ will be as large as there are many paths made of edges that correspond to large values in $A$ from a node to itself in $k$ steps.

In comparison to our work, \citet{DAG-1} use a quadratic loss on the corresponding linear structural equation model (SEM) as the score function $F(A)$. By attaching normalizing flows to topology learning, our method has a continuously adjustable level of complexity and does not make any strong assumptions on the form of the conditional factors. 

\subsection{Stochastic adjacency matrix}
In order to learn the BN topology from the data, the adjacency matrix must be relaxed to contain reals instead of booleans. It also implies that the graph induced by $A$ does not formally define a DAG during training. Work using NO TEARS to perform topology learning directly plug the real matrix $A$ in \eqref{eq:DAG-relax} \citep{DAG-1, DAG-2, DAG-3, DAG-4} however this is inadequate because the quantity of information going from node $j$ to node $i$ does not continuously relate to the value of $A_{i,j}$. Either the information is null if $A_{i,j} = 0$ or it passes completely if not. Instead, 
we propose to build the stochastic pseudo binary valued matrix $A'$ from $A$, defined as
$$A'_{i, j} = \frac{e^{\frac{\log(\sigma(A_{i,j}^2)) + \gamma_1}{T}}}{e^{\frac{\log(\sigma(A_{i,j}^2)) + \gamma_1}{T}} + e^{\frac{\log(1 - \sigma(A_{i,j}^2)) + \gamma_2}{T}}}, $$
where $\gamma_1, \gamma_2 \sim \text{Gumbel}(0, 1)$ and $\sigma(a) = 2(\text{sigmoid}(2a^2) - \frac{1}{2})$ normalizes the values of $A$ between $0$ and $1$, being close to $1$ for large values and close to zero for values close to $0$. The hyper-parameter $T$ controls the sampling temperature and is fixed to $0.5$ in all our experiments. 
In contrast to directly using the matrix $A$, this stochastic transformation referred to as the Gumbel-Softmax trick in the literature \citep{gumbel1, gumbel2} allows to create a direct and continuously differentiable relationship between the weights of the edges and the quantity of information that can transit between two nodes. Indeed, the probability mass of the random variables $A'_{i,j}$ is mainly located around $0$ and $1$, and its expected value converges to $1$ when $A_{i,j}$ increases. 

\subsection{Optimization} \label{sec:optim}
We rely on the augmented Lagrangian approach to solve the constrained optimization problem \eqref{eq:DAG-relax} as initially proposed by \citet{DAG-1}. This optimization procedure requires solving iteratively the following sub-problems:
\begin{align}
    \max_{A} \mathbb{E}_{\gamma_1, \gamma_2}\left[F(A)\right] - \lambda_t w(A) - \frac{\mu_t}{2} w(A)^2, \label{eq:loss}
\end{align}
where $\lambda_t$ and $\mu_t$ respectively denote the Lagrangian multiplier and penalty coefficients of the sub-problem $t$.

 We solve these optimization problems with mini-batch stochastic gradient ascent. We update the values of $\gamma_t$ and $\mu_t$ as suggested by \cite{DAG-2} when the validation loss does not improve for 10 consecutive epochs. Once $w(A)$ equals $0$, the adjacency matrix is acyclic up to numerical errors. We recover an exact DAG by thresholding the elements of $A$ while checking for acyclicity with a path finding algorithm. 
We provide additional details about the optimization procedure used in our experiments in Appendix~\ref{app:optim}.

\section{Experiments}
In this section, we demonstrate some applications of graphical NFs in addition to unifying NFs and BN under a common framework. We first demonstrate how pre-loading a known or hypothesized DAG structure can help finding an accurate distribution of the data. Then, we show that learning the graph topology leads to relevant BNs that support generalization well when combined with $\ell_1$-penalization. Finally, we demonstrate that mono-step normalizing flows made of graphical conditioners are competitive density estimators. 
\subsection{On the importance of graph topology}
The following experiments are performed on four distinct datasets, three of which are synthetic, such that we can define a ground-truth minimal Bayesian network, and the fourth is a causal protein-signaling network derived from single-cell data \citep{proteins}. Additional information about the datasets are provided in \tabref{tab:ds_desc} and in Appendix \ref{app:top_dataset}. For each experimental run we first randomly permute the features before training the model in order to compare autoregressive and graphical conditioners fairly.
\paragraph{Prescribed topology}
\begin{table}
    \caption{Datasets description. $d$=Dimension of the data. $V$=Number of edges in the ground truth Bayesian Network.} \label{tab:ds_desc}
    \centering
    \scriptsize
    \setlength{\tabcolsep}{2pt}
    \renewcommand{\arraystretch}{1.5}
    
    \begin{tabular}{l l l l l}
        \hline\hline
        Dataset & $d$ & $V$ & Train & Test\\ \hline
        \tbf{Arithmetic Circuit} & $8$ & $8$ & $10,000$ & $5,000$\\
        \tbf{8 Pairs}& $16$ & $8$  & $10,000$ & $5,000$\\
        \tbf{Tree}  & $7$ & $8$ & $10,000$ & $5,000$\\
        \tbf{Protein} & $11$ & $20$ & $6,000$ & $1,466$\\
        \hline
        \tbf{POWER} & $6$ & $\leq 15$ & $1,659,917$ & $204,928$\\
        \tbf{GAS} & $8$ & $\leq 28$ & $852,174$ & $105,206$\\
        \tbf{HEPMASS} & $21$ & $\leq 210$ & $315,123$ & $174,987$\\
        \tbf{MINIBOONE} & $43$ & $\leq 903$ & $29,556$ & $3,648$\\
        \tbf{BSDS300} & $63$ & $\leq 1,953$ & $1,000,000$ & $250,000$\\
        \hline \hline
    \end{tabular}
\end{table}
\begin{table}
    \caption{Graphical vs autoregressive conditioners combined with monotonic normalizers. Average log-likelihood on test data over 5 runs, under-scripted error bars are equal to the standard deviation.
    Results are reported in nats; higher is better. 
    The best performing architecture for each dataset is written in bold. \emph{Graphical conditioners clearly lead to improved density estimation when given a relevant prescribed topology in 3 out of the 4 datasets.}} \label{tab:known_topo}
    \centering
    \scriptsize
    \setlength{\tabcolsep}{1pt}
    \renewcommand{\arraystretch}{1.5}
    
    \begin{tabular}{l c c c c}
        \hline\hline
        Conditioner & Graphical & Autoreg. \\  \hline
        \tbf{Arithmetic Circuit} & \bestresult{3.986817}{0.155741} & \result{3.059781}{0.381778} \\
        \tbf{8 Pairs}& \bestresult{-9.398858}{0.061945} & \result{-11.503608}{0.271145}\\
        \tbf{Tree}  & \bestresult{-6.845759}{0.016316} & \result{-6.960843}{0.054900}\\
        \tbf{Protein} & \result{6.460093}{0.075050} & \bestresult{7.516195}{0.099126}\\
        \hline
        \hline
    \end{tabular}
    \vspace{-1em}
\end{table}

Rarely do real data come with their associated Bayesian network however oftentimes experts want to hypothesize a network topology and to rely on it for the downstream tasks. Sometimes the topology is  known a priori, as an example the sequence of instructions in stochastic simulators can usually be translated into a graph topology (e.g. in probabilistic programming \citep{proba_prog, wood}). In both cases, graphical conditioners allow to explicitly take advantage of this to build density estimators while keeping the assumptions made about the network topology valid. 

\tabref{tab:known_topo} presents the test likelihood of autoregressive and graphical normalizing flows on the four datasets. The flows are made of a single step and use monotonic normalizers. The neural network architectures are all identical. Further details on the experimental settings as well as additional results for affine normalizers are respectively provided in Appendix \ref{app:top_dataset_descri} and \ref{app:top_dataset_exp}. \emph{We observe how using correct BN structures lead to good test performance in \tabref{tab:known_topo}.} 
Surprisingly, the performance on the protein dataset are not improved when using the ground truth graph. We observed during our experiments (see Appendix \ref{app:top_dataset_exp}) that learning the topology from this dataset sometimes led to improved density estimation performance with respect to the ground truth graph. The limited dataset size does not allow us to answer if this comes from to the limited capacity of the flow and/or from the erroneous assumptions in the ground truth graph. However, we stress out that the graphical flow respects the assumed causal structure in opposition to the autoregressive flow.


\paragraph{Learning the topology}
\begin{figure}
    \begin{subfigure}{.27\textwidth}
    \includegraphics[width=.99\textwidth]{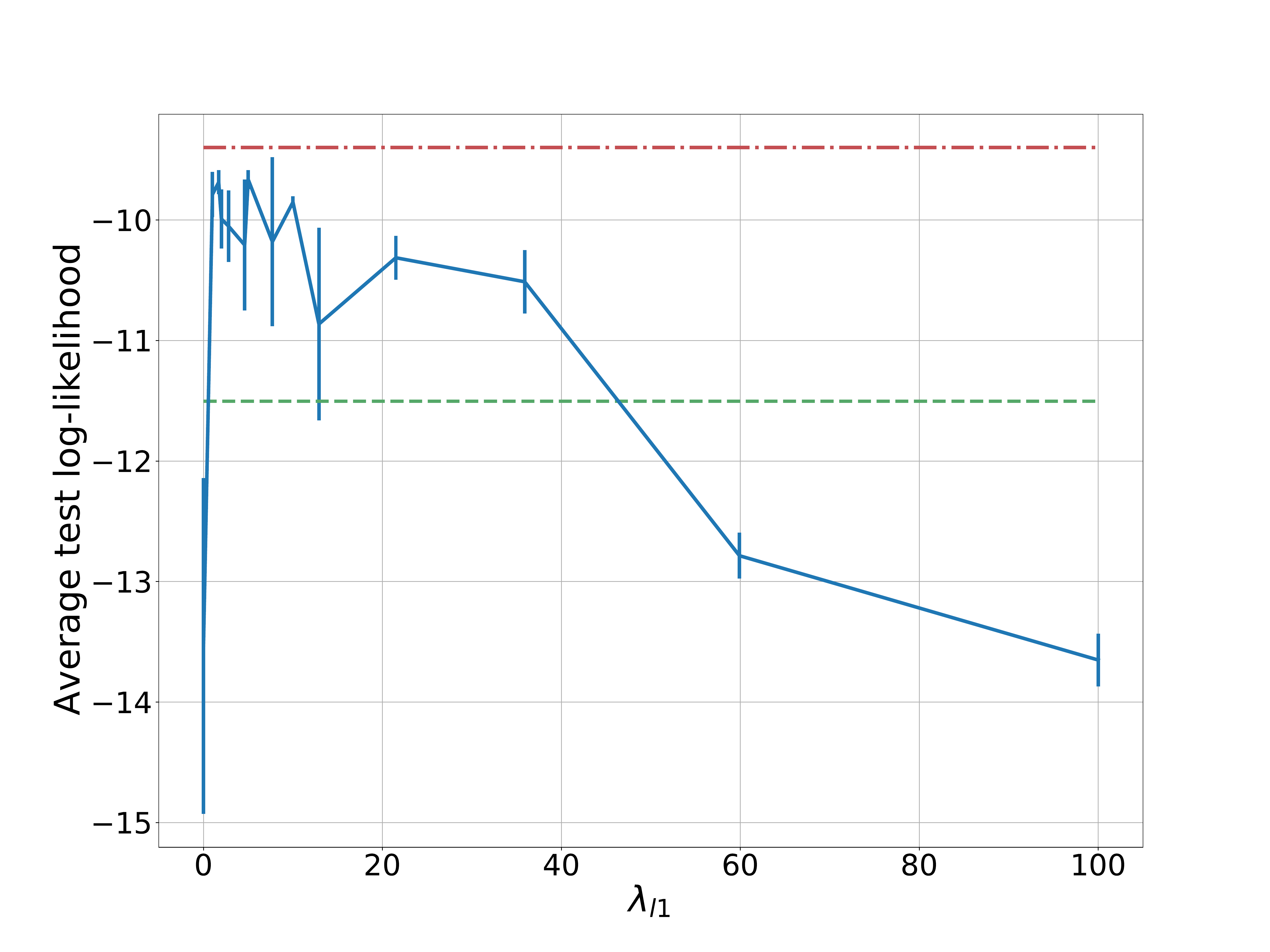}
    \caption{}\label{fig:8_MIX_L1_test}
    \end{subfigure}
    \begin{subfigure}{.2\textwidth}~
    \centering
        \includegraphics[width=1.\textwidth]{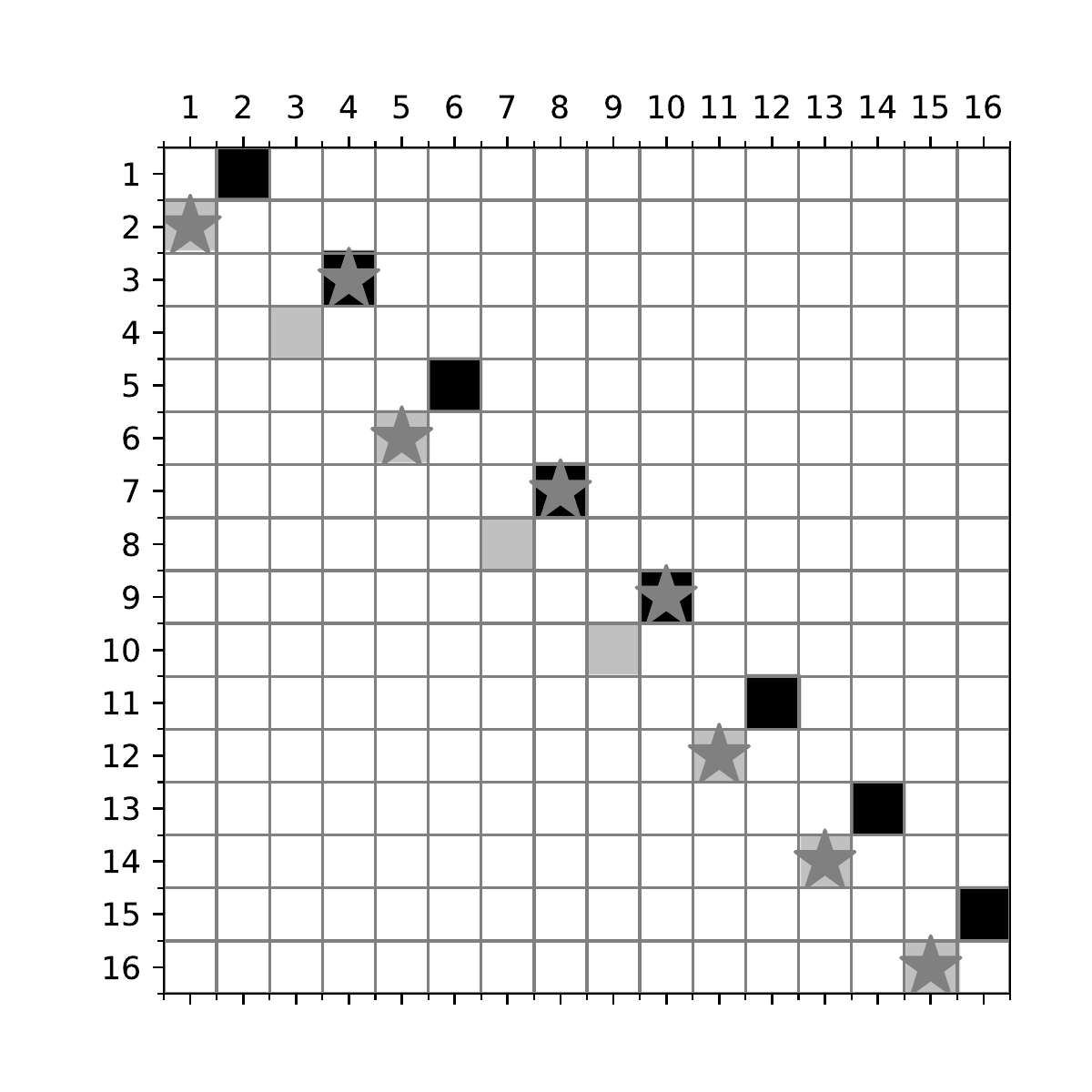}
    \caption{}\label{fig:8_MIX_L1_DAG}
    \end{subfigure}
    \caption{\tbf{(a)}: Test log-likelihood as a function of $\ell_1$-penalization on \tbf{8 pairs} dataset. The upper bound is the average result when given a prescribed topology, the lower bound is the result with an autoregressive conditioner. \emph{Learning the right topology leads to better results than autoregressive conditioners.} \tbf{(b)}: The blacked cells corresponds to one correct topology of the 8 pairs dataset and the grey cells to the transposed adjacency matrix. The stars denote the edges discovered by the graphical conditioner when trained with $\lambda_{\ell_1} = 4$. \emph{The optimization discovers a relevant BN (equivalent to the ground truth)}.}
    \label{fig:8_MIX_L1}
\end{figure}
\begin{figure}
    \centering
    \includegraphics[width=.2\textwidth]{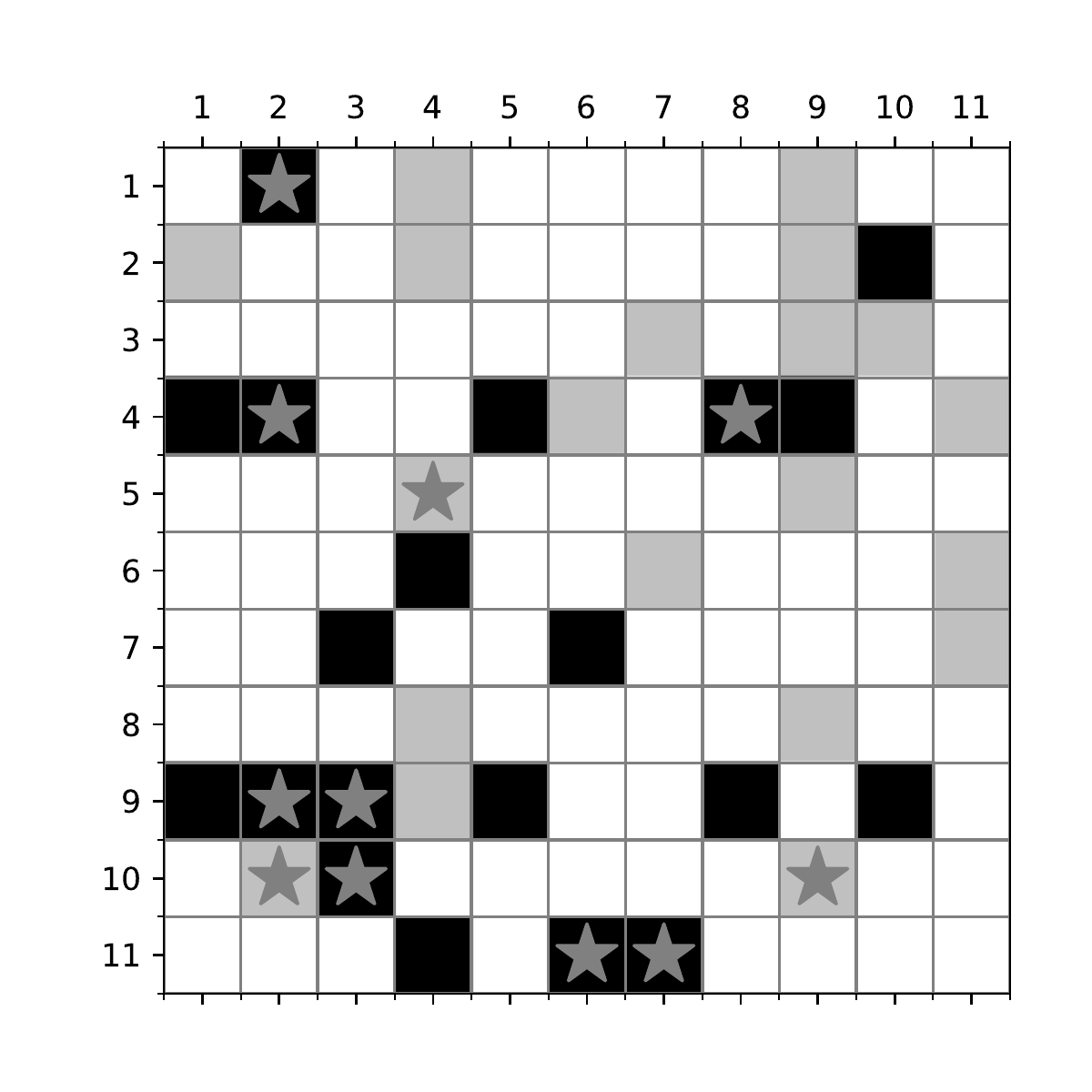}
    \caption{The adjacency matrix of the protein interaction network. The blacked cells are the directed connections proposed by domain experts and the grey is the transposed adjacency matrix. The stars denote the edges discovered by the graphical conditioner when trained with  $\lambda_{\ell_1} = 12$. \emph{In a realistic setting, the optimization leads to a graph that shares a lot with the one designed by experts.}}    \label{fig:protein_net}

\end{figure}
The $\lambda_{\ell_1}$ coefficient introduced in \eqref{eq:score} controls the sparsity of the optimized BN, allowing to avoid some spurious connections. We now analyze the effect of sparsity on the density estimation performance. \figref{fig:8_MIX_L1_test} shows the test log-likelihood as a function of the $\ell_1$-penalization on the 8 pairs dataset. We observe that the worst results are obtained when there is no penalization. Indeed, in this case the algorithm finds multiple spurious connections between independent vector's components and then overfits on these incorrect relationships. Another extreme case shows up when the effect of penalization is too strong, in this case, the learning procedure underfits the data because it ignores too many relevant connections. It can also be concluded from the plot that the optimal $\ell_1$-penalization performs on par with the ground truth topology, and certainly better than the autoregressive conditioner. Additional results on the other datasets provided in Appendix \ref{app:top_dataset} lead to similar conclusions.

\paragraph{Protein network}
The adjacency matrix discovered by optimizing \eqref{eq:loss} (with $ \lambda_{\ell_1} = 12$) on the protein network dataset is shown in \figref{fig:protein_net}. All the connections discovered by the method correspond to ground truth connections. Unsurprisingly, their orientation do not always match with what is expected by the experts. Overall we can see that the optimization procedure is able to find relevant connections between variables and avoids spurious ones when the $\ell_1$-penalization is optimized. Previous work on topology learning such as \cite{DAG-1,DAG-2, DAG-3} compare their method to others by looking at the topology discovered on the same dataset. Here we do not claim that learning the topology with a graphical conditioner improves over previous methods. Indeed, we believe that the difference between the methods mainly relies on the hypothesis and/or inductive biased made on the conditional densities. Which method would perform better is dependent on the application. Moreover, the Structural Hamiltonian Distance (SHD) and the Structural Inference Distance (SID), often used to compare BNs, are not well motivated in general. On the one hand the SHD does not take into account the independence relationships modeled with Bayesian networks, as an example the SHD between the graph found on 8 pairs dataset of \figref{fig:8_MIX_L1_DAG} and one possible ground truth is non zero whereas the two BNs are equivalent. On the other hand the SID aims to compare causal graphs whereas the methods discussed here do not learn causal relationships. In general BN topology identification is an ill posed problem and thus we believe using these metrics to compare different methods without additional downstream context is irrelevant. However, we conclude from \figref{fig:8_MIX_L1_DAG} and \figref{fig:protein_net} that the scores computed with graphical NFs can be used to learn relevant BN structures.

\subsection{Density estimation benchmark} \label{sec:exp-tabular-density}
\begin{table}

    \caption{Average log-likelihood on test data over 3 runs, under-scripted error bars are equal to the standard deviation.
    Results are reported in nats; higher is better. 
    The best performing architecture per category for each dataset is written in bold. 
    (a) 1-step affine normalizers (b) 1-step monotonic normalizers. \emph{Graphical normalizing flows outperform coupling and autoregressive architectures on most of the datasets.}} \label{tab:tabular_density}
    \centering
    \scriptsize
    \setlength{\tabcolsep}{0pt}
    \renewcommand{\arraystretch}{1.5}
    
    \begin{tabular}{l l  c c c c c}
        \hline\hline
        &\scalebox{.8}{Dataset} & \scalebox{.6}{\tbf{POWER}} & \scalebox{.6}{\tbf{GAS}} & \scalebox{.6}{\tbf{HEPMASS}} & \scalebox{.6}{\tbf{MINIBOONE}} & \scalebox{.6}{\tbf{BSDS300}} \\
        \hline
        
        \multirow{3}{*}{\scalebox{.8}{(a)}}
        & \scalebox{.8}{Coup. }
            & \result{-5.595652333333334}{0.002610780002646944}
            & \result{-3.0495433333333337}{0.007895648435829896}
            & \result{-25.741226333333334}{0.006650417346971497}
            & \result{-38.34233366666667}{0.024608616164984692}
            & \result{57.334731}{0.0024840547497991497}\\ 
        & \scalebox{.8}{Auto.}
            & \result{-3.552735333333333}{0.001566747019215928}
            & \result{-0.34099799999999997}{0.008537987702029085}
            & \result{-21.661752333333336}{0.007925831074544806}
            & \bestresult{-16.6966}{0.05012842471758638}
            & \bestresult{63.73638033333333}{0.004330463049399384} \\ 
        & \scalebox{.8}{Graph. }
            & \bestresult{-2.7979636666666665}{0.00799790907397394}
            & \bestresult{1.985552666666667}{0.020352012665963974}
            & \bestresult{-21.18043833333333}{0.07373030467106993}
            & \result{-19.672145}{0.05610214543134648}
            & \result{62.85277466666667}{0.06672677328396033}\\ \hline
        \multirow{3}{*}{\scalebox{.8}{(b)}}
        & \scalebox{.8}{Coup. }
            & \result{0.246356}{0.0012096316243661454}
            & \result{5.121069666666667}{0.029769947064484863}
            & \result{-20.545104333333335}{0.04462281262115015}
            & \result{-32.04131233333334}{0.12106401195327407}
            & \result{107.17215866666668}{0.4562053354454408} \\
        & \scalebox{.8}{Auto.}
            & \result{0.5769643333333333}{0.0017583917525840062}
            & \result{9.786357666666667}{0.042050408403353966}
            & \result{-14.515809666666668}{0.15825141180061844}
            & \bestresult{-11.65676}{0.017674348323677303}
            & \result{151.28819533333333}{0.3101846263988979}  \\
        & \scalebox{.8}{Graph. }
            & \bestresult{0.620236}{0.043806243550434706}
            & \bestresult{10.146548666666666}{0.1515902526183299}
            & \bestresult{-14.165973333333334}{0.13134025332784377}
            & \result{-16.227559}{0.5166131106860017}
            & \bestresult{155.22247133333335}{0.10754154750090603} \\
       \hline \hline
    \end{tabular}
\end{table}

In these experiments, we compare autoregressive, coupling and graphical conditioners with no $\ell_1$-penalization on benchmark tabular datasets as introduced by \cite{MAF} for density estimation. See \tabref{tab:ds_desc} for a description. We evaluate each conditioner in combination with monotonic and affine normalizers.
We only compare NFs with a single transformation step because our focus is on the conditioner capacity. We observed during preliminary experiments that stacking multiple conditioners  improves the performance slightly, however the gain is marginal compared to the loss of interpretability. To provide a fair comparison we have fixed in advance the neural architectures used to parameterize the normalizers and conditioners as well as the training parameters by taking inspiration from those used by \citet{UMNN} and \citet{MAF}. The variable ordering of each dataset is randomly permuted at each run. All hyper-parameters are provided in Appendix \ref{app:archi-tabular} and a public implementation will be released on Github. 

First, \tabref{tab:tabular_density} presents the test log-likelihood obtained by each architecture. These results indicate that graphical conditioners offer the best performance in general. Unsurprisingly, coupling layers show the worst performance, due to the arbitrarily assumed independencies. Autoregressive and graphical conditioners show very similar performance for monotonic normalizers, the latter being slightly better on 4 out of the 5 datasets. \tabref{tab:tabular_density_all} contextualizes the performance of graphical normalizing flows with respect to the most popular normalizing flow architectures.
Comparing the results together, we see that while additional steps lead to noticeable improvements for affine normalizers (MAF), benefits are questionable for monotonic transformations. Overall, graphical normalizing flows made of a single transformation step are competitive with the best flow architectures with the added value that they can directly be translated into their equivalent BNs. \emph{From these results we stress out that single step graphical NFs are able to model complex densities on par with SOTA while they offer new ways of introducing domain knowledge.}

Second, \tabref{tab:edges} presents the number of edges in the BN associated with each flow. For \textsc{power} and \textsc{gas}, the number of edges found by the graphical conditioners is close or equal to the maximum number of edges. Interestingly, graphical conditioners outperform  autoregressive conditioners on these two tasks, demonstrating the value of finding an appropriate ordering particularly when using affine normalizers. Moreover, graphical conditioners correspond to BNs whose sparsity is largely greater than for autoregressive conditioners while providing equivalent if not better performance. The depth \citep{depth-dag} of the equivalent BN directly limits the number of steps required to inverse the flow. Thus sparser graphs that are inevitably shallower correspond to NFs for which sampling, i.e. computing their inverse, is faster.
\begin{table}
    \caption{Average log-likelihood on test data over 3 runs, under-scripted error bars are equal to the standard deviation. Results are reported in nats, higher is better. The results followed by a star are copied from the literature and the number of steps in the flow is indicated in parenthesis for each architecture. \textit{Graphical normalizing flows reach density estimation performance on par with the most popular flow architectures whereas it is only made of 1 transformation step.}} \label{tab:tabular_density_all}
    \centering
    \scriptsize
    \setlength{\tabcolsep}{0pt}
    \renewcommand{\arraystretch}{1.5}
    
    \begin{tabular}{l  c c c c c}
        \hline\hline
        \scalebox{.8}{Dataset} & \scalebox{.6}{\tbf{POWER}} & \scalebox{.6}{\tbf{GAS}} & \scalebox{.6}{\tbf{HEPMASS}} & \scalebox{.6}{\tbf{MINIBOONE}} & \scalebox{.6}{\tbf{BSDS300}} \\
        \hline
        \scalebox{.7}{Graph.-UMNN (1)}
            & \result{0.620236}{0.043806243550434706}
            & \result{10.146548666666666}{0.1515902526183299}
            & \result{-14.165973333333334}{0.13134025332784377}
            & \result{-16.227559}{0.5166131106860017}
            & \result{155.22247133333335}{0.10754154750090603} \\
            \hline
        \scalebox{.7}{MAF (5)}
            & \result{0.14}{0.005} 
            & \result{9.07}{0.01}
            & \result{-17.7}{0.01}
            & \result{-11.75}{0.22}
            & \result{155.69}{0.14} \\ 
        \scalebox{.7}{Glow$^\star$ (10)}
            & \result{0.42}{0.01} 
            & \result{12.24}{0.03}
            & \result{-16.99}{0.02}
            & \result{-10.55}{0.45}
            & \result{156.95}{0.28} \\ 
        \scalebox{.7}{UMNN-MAF$^\star$ (5)}
            & \result{0.63}{0.01}
            & \result{10.89}{0.7}
            & \result{-13.99}{.21}
            & \result{-9.67}{.13}
            & \result{157.98}{.01}
            \\         
        \scalebox{.7}{Q-NSF$^\star$ (10)}
            & \result{0.66}{0.005}
            & \result{12.91}{0.01}
            & \result{-14.67}{.02}
            & \result{-9.72}{.24}
            & \result{157.42}{.14}
            \\ 
        \scalebox{.7}{FFJORD$^\star$ (5-5-10-1-2)}
            & \result{0.46}{0.01} 
            & \result{8.59}{0.12} 
            & \result{-14.92}{0.08} 
            & \result{-10.43}{0.04} 
            & \result{157.40}{0.19}
            \\ \hline \hline
    \end{tabular}
\end{table}

\begin{table}
    \caption{Rounded average number of edges (over 3 runs) in the equivalent Bayesian network. \textit{The graphical conditioners lead to sparser BNs compared to autoregressive conditioners.}} \label{tab:edges}
    \centering
    \scriptsize
    \setlength{\tabcolsep}{2pt}
    \renewcommand{\arraystretch}{1.5}
    
    \begin{tabular}{l c c c c c}
        \hline\hline
        Dataset & \tbf{P} & \tbf{G} & \tbf{H} & \tbf{M} & \tbf{B} \\
        \hline
        Graph.-Aff. & $15$   & $26$  & $152$  & $277$  & $471$  \\
        Graph.-Mon. & $15$    & $27$  & $159$  & $265$  & $1594$  \\
        \hline
        Coupling & $9$  & $16$  & $110$  & $462$ & $992$ \\ 
        Autoreg. & $15$  & $28$  & $210$  & $903$ & $1953$\\ 
         \hline\hline
        
    \end{tabular}
\end{table}

    
        
    

\section{Discussion}
\paragraph{Cost of learning the graph structure}
The attentive reader will notice that learning the topology does not come for free. Indeed, the Lagrangian formulation requires solving a sequence of optimization problems which increases the number of epochs before convergence. In our experiments we observed different overheads depending on the problems, however in general the training time is at least doubled. This does not impede the practical interest of using graphical normalizing flows. The computation overhead is more striking for non-affine normalizers (e.g. UMNNs) that are computationally heavy. However, we observed that most of the time the topology recovered by affine graphical NFs is relevant. It can thus be used as a prescribed topology for normalizers that are heavier to run, hence alleviating the computation overhead. Moreover, one can always hypothesize on the graph topology but more importantly the graph learned is usually sparser than an autoregressive one while achieving similar if not better results. The sparsity is interesting for two reasons: it can be exploited for speeding up the forward density evaluation; but more importantly it usually corresponds to shallower BNs that can be inverted faster than autoregressive structures. 
\paragraph{Bayesian network topology learning}
Formal BN topology learning has extensively been studied for more than 30 years now and many strong theoretical results on the computational complexity have been obtained. Most of these results however focus on discrete random variables, and how they generalize in the continuous case is yet to be explained. The topic of BN topology learning for discrete variables has been proven to be NP-hard by \citet{chickering-NP-hard}. However, while some greedy algorithms exist, they do not lead in general to a minimal I-map although allowing for an efficient factorization of random discrete vectors distributions in most of the cases. These algorithms are usually separated between the constrained-based family such as the \textit{PC algorithm}~\citep{pc-algorithm} or the \textit{incremental association Markov blanket}~\citep{PGM-book} and the score-based family as used in the present work.
Finding the best BN topology for continuous variables has not been proven to be NP-hard however the results for discrete variables suggest that without strong assumptions on the function class the problem is hard. 

The recent progress made in the continuous setting relies on the heuristic used in score-based methods. In particular, \citet{DAG-1} showed that the acyclicity constraint required in BNs can be expressed with NO TEARS, as a continuous function of the adjacency matrix, allowing the Lagrangian formulation to be used.
\citet{DAG-2} proposed DAG-GNN, a follow up work of \citet{DAG-1} which relies on variational inference and auto-encoders to generalize the method to non-linear structural equation models. Further investigation of continuous DAG learning in the context of causal models was carried out by \citet{DAG-3}. They use the adjacency matrix of the causal network as a mask over neural networks to design a score which is the log-likelihood of a parameterized normal distribution. The requirement to pre-define a parametric distribution before learning restricts the factors to simple conditional distributions. In contrast, our method combines the constraints given by the BN topology with  NFs which are free-form universal density estimators. Remarkably, their method leads to an efficient one-pass computation of the joint density. This neural masking scheme can also be implemented for NF architectures such as already demonstrated by \citet{MAF} and \citet{BNAF} for autoregressive conditioners. 

\paragraph{Shuffling between transformation steps}
As already mentioned, consecutive transformation steps are often combined with randomly fixed permutations in order to mitigate the ordering problem. Linear flow steps \citep{TAN} and 1x1 invertible convolutions \citep{GLOW} generalize these fixed permutations. They are parameterized by a matrix $W = PLU$ where $P$ is the fixed permutation matrix, and $L$ and $U$ are respectively a lower and an upper triangular matrix. Although linear flow improves the simple permutation scheme, they do still rely on an arbitrary permutation. To the best of our knowledge, graphical conditioners are the first attempt to get completely rid of any fixed permutation in NFs.


\paragraph{Inductive bias}
Graphical conditioners eventually lead to binary masks that model the conditioning components of a factored joint distribution. In this way, the conditioners process their input as they would process the full vector.
We show experimentally in Appendix \ref{app:MNIST} that this effectively leads to good inductive bias for processing images with NFs. 
In addition, we have shown that normalizing flows built from graphical conditioners combined with monotonic transformations are expressive density estimators. 
In effect, this means that enforcing some a priori known independencies can be performed thanks to graphical normalizing flows without hurting their modeling capacity. We believe such models could be of high practical interest because they cope well with large datasets and complex distributions while preserving some readability through their equivalent BN. 

Close to our work, \citet{wood} improve amortized inference by prescribing a BN structure between the latent and observed variables into a FFJORD NF, once again showing the interest of using the potential BN knowledge. Similar to our work, \citet{khemakhem2020causal} see causal autoregressive flows as structural equation modelling. They show bivariate autoregressive affine flows can be used to identify the causal direction under mild conditions. Under similar mild conditions, discovering causal relationships with graphical normalizing flows could well be an exciting research direction.




\paragraph{Conclusion}
We have revisited coupling and autoregressive conditioners for normalizing flows as Bayesian networks. 
From this new perspective, we proposed the more general graphical conditioner architecture for normalizing flows. We have illustrated the importance of assuming or learning a relevant Bayesian network topology for density estimation. In addition, we have shown that this new architecture compares favorably with autoregressive and coupling conditioners and on par to the most common flow architectures on standard density estimation tasks even without any hypothesized topology. One interesting and straightforward extension of our work would be to combine it with normalizing flows for discrete variables.
We also believe that graphical conditioners could be used when the equivalent Bayesian network is required for downstream tasks such as in causal reasoning.
\subsubsection*{Acknowledgments}
We thank Vân Anh Huynh-Thu, Johann Brehmer, and Louis Wehenkel for proofreading this manuscript or an earlier version of it. We are also thankful to the reviewers for helpful comments. Antoine Wehenkel is a research fellow of the F.R.S.-FNRS (Belgium) and acknowledges its financial support. Gilles Louppe is recipient of the ULiège - NRB Chair on Big data and is
thankful for the support of NRB.
\bibliography{biblio}
\clearpage
\appendix
\onecolumn
\appendix
\section{Optimization procedure}\label{app:optim}
\begin{algorithm}
\begin{algorithmic} \caption{Main Loop}
\State $\text{epoch} \gets 0$
\While {$!\text{Stopping criterion}$} 
    \ForAll{batch $\mb{X} \in \mb{X_{train}}$}
        \State $\text{loss} \gets \Call{computeLoss}{\text{flow}, X}$
        \State $\Call{optimize}{\text{flow}, \text{loss}}$
    \EndFor
   \State $\text{loss}_{valid} \gets \Call{computeLoss}{\text{flow}, \mb{X_{test}}}$
    \State $\text{epoch} \gets \text{epoch} + 1$
    \State $\Call{updateCoefficients}{\text{flow}, \text{epoch}, \text{loss}_{valid}}$
    \If{$\Call{isDagConstraintNull}{\text{flow}}$}
        \State $\Call{PostProcess}{\text{flow}}$
    \EndIf
\EndWhile
\end{algorithmic}
\end{algorithm}
The method $\Call{computeLoss}{\text{flow}, X}$ is computed as described by equation~\eqref{eq:loss}. The $\Call{optimize}{\text{flow}, \text{loss}}$ method performs a backward pass and an optimization step with the chosen optimizer (Adam in our experiments). The post-processing is peformed by $\Call{PostProcess}{\text{flow}}$ and consists in thresholding the values in $A$ such that the values below a certain threshold are set to $0$ and the other values to $1$, after post-processing the stochastic door is deactivated. The threshold is the smallest real value that makes the equivalent graph acyclic. The method $\Call{updateCoefficients}$ updates the Lagrangian coefficients as described in section \ref{sec:optim}.

\section{Jacobian of graphical conditioners}\label{app:proof-jac-gnf}
\begin{proposition}\label{app:prop:jac-gnf}
The absolute value of the determinant of the Jacobian of a normalizing flow step based on graphical conditioners is equal to the product of its diagonal terms. 
\end{proposition}
\begin{proof} \textbf{Proposition \ref{app:prop:jac-gnf}}
A Bayesian Network is a directed acyclic graph. \cite{sedgewick2011algorithms} showed that every directed acyclic graph has a topological ordering, it is to say an ordering of the vertices such that the starting endpoint of every edge occurs earlier in the ordering than the ending endpoint of the edge. Let us suppose that an oracle gives us the permutation matrix $P$ that orders the components of $\mb{g}$ in the topological defined by $A$. Let us introduce the following new transformation $\mb{g}_P(\mb{x}_P) = P\mb{g}(P
^{-1}(P\mb{x}))$ on the permuted vector $\mb{x}_P = P\mb{x}$. The Jacobian of the transformation $\mb{g}_P$ (with respect to $\mb{x}_P$) is lower triangular with diagonal terms given by the derivative of the normalizers with respect to their input component. The determinant of such Jacobian is equal to the product of the diagonal terms. Finally, we have
\begin{align*}
    |\det(J_{\mb{g}_P(\mb{x}_P)})| &= |\det(P)||\det(J_{\mb{g}(\mb{x})})| \frac{|\det(P)|}{|\det(P)|}\\
    &= |\det(J_{\mb{g}(\mb{x})})|,
\end{align*}
because of (1) the chain rule; (2) The determinant of the product is equal to the product of the determinants; (3) The determinant of a permutation matrix is equal to $1$ or $-1$.
The absolute value of the determinant of the Jacobian of $\mb{g}$ is equal to the absolute value of the determinant of $\mb{g}_P$, the latter given by the product of its diagonal terms that are the same as the diagonal terms of $\mb{g}$. Thus the absolute value of the determinant of the Jacobian of a normalizing flow step based on graphical conditioners is equal to the product of its diagonal terms.
\end{proof}
\section{Experiments on topology learning}\label{app:top_dataset}
\subsection{Neural networks architecture}
We use the same neural network architectures for all the experiments on the topology. The conditioner functions $h_i$ are modeled by shared neural networks made of 3 layers of $100$ neurons. When using UMNNs for the normalizer we use an embedding size equal to $30$ and a 3 layers of $50$ neurons MLP for the integrand network.
\subsection{Dataset description}\label{app:top_dataset_descri}
\begin{figure*}
\caption{Ground truth adjacency matrices. Black squares denote direct connections and in light grey is their transposed.}
\centering
\begin{subfigure}[t]{.2\textwidth}
\includegraphics[width=1.\textwidth]{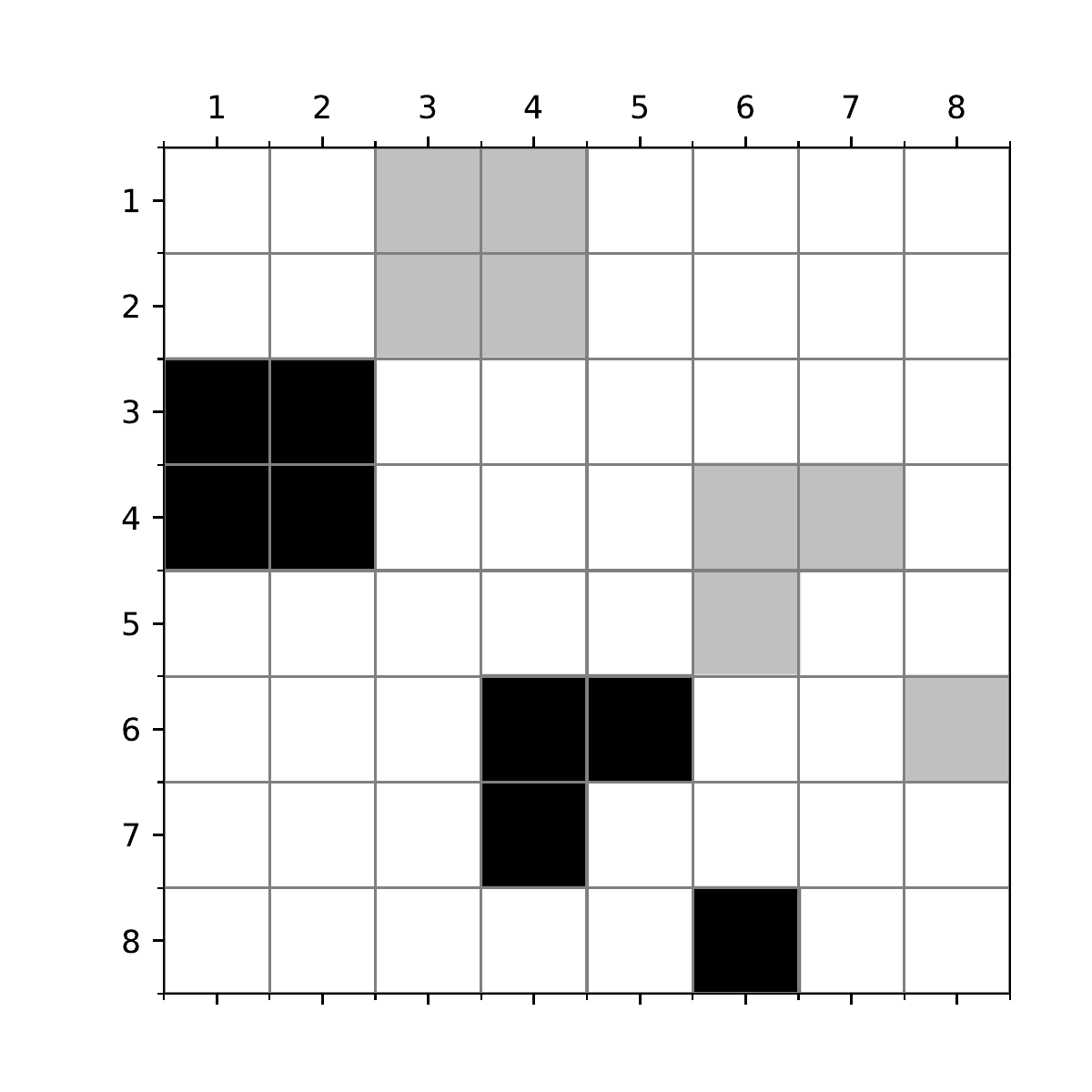}
\caption{Arithmetic Circuit}
\end{subfigure}
\begin{subfigure}[t]{.2\textwidth}
\includegraphics[width=1.\textwidth]{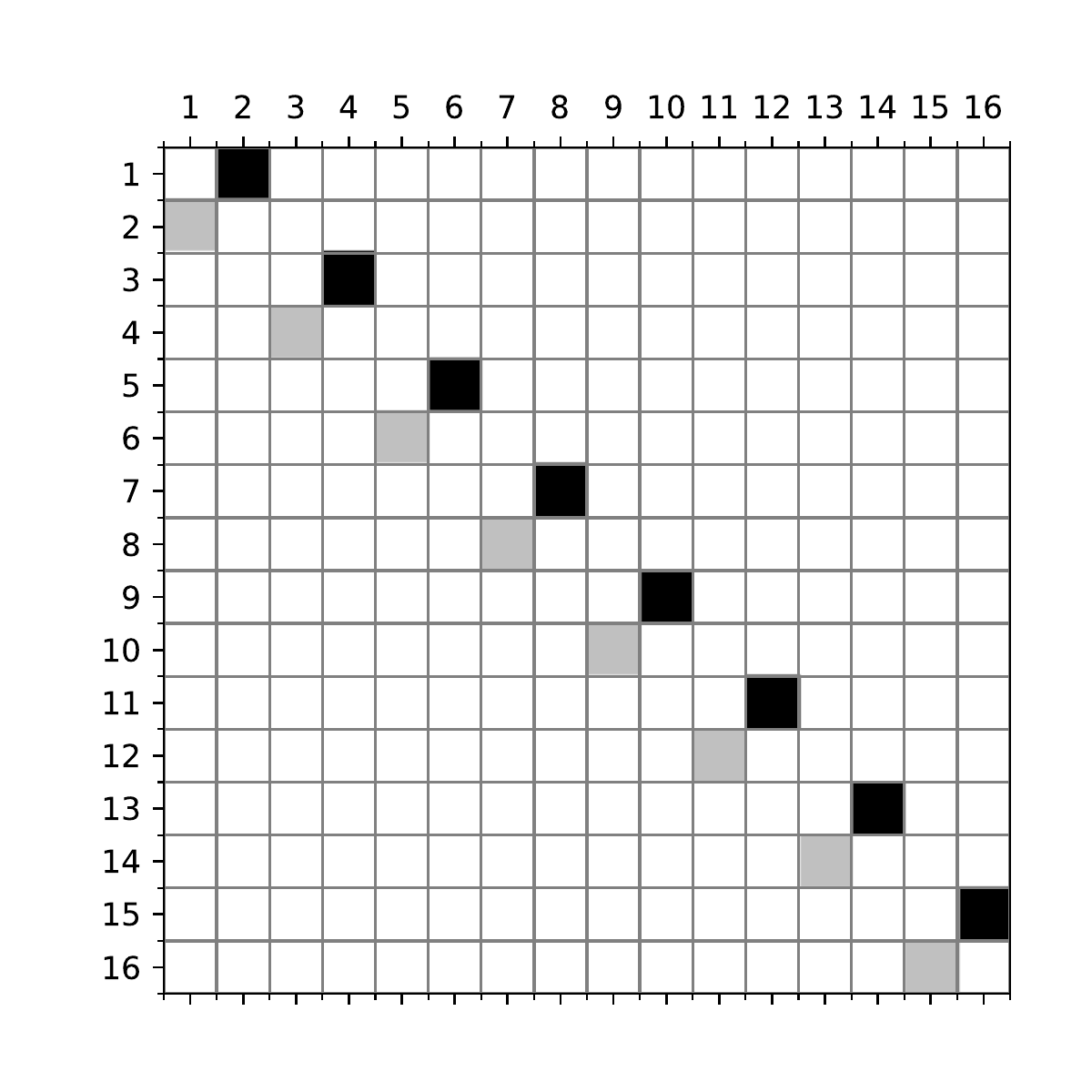}
\caption{8 Pairs}
\end{subfigure}
\begin{subfigure}[t]{.2\textwidth}
\includegraphics[width=1.\textwidth]{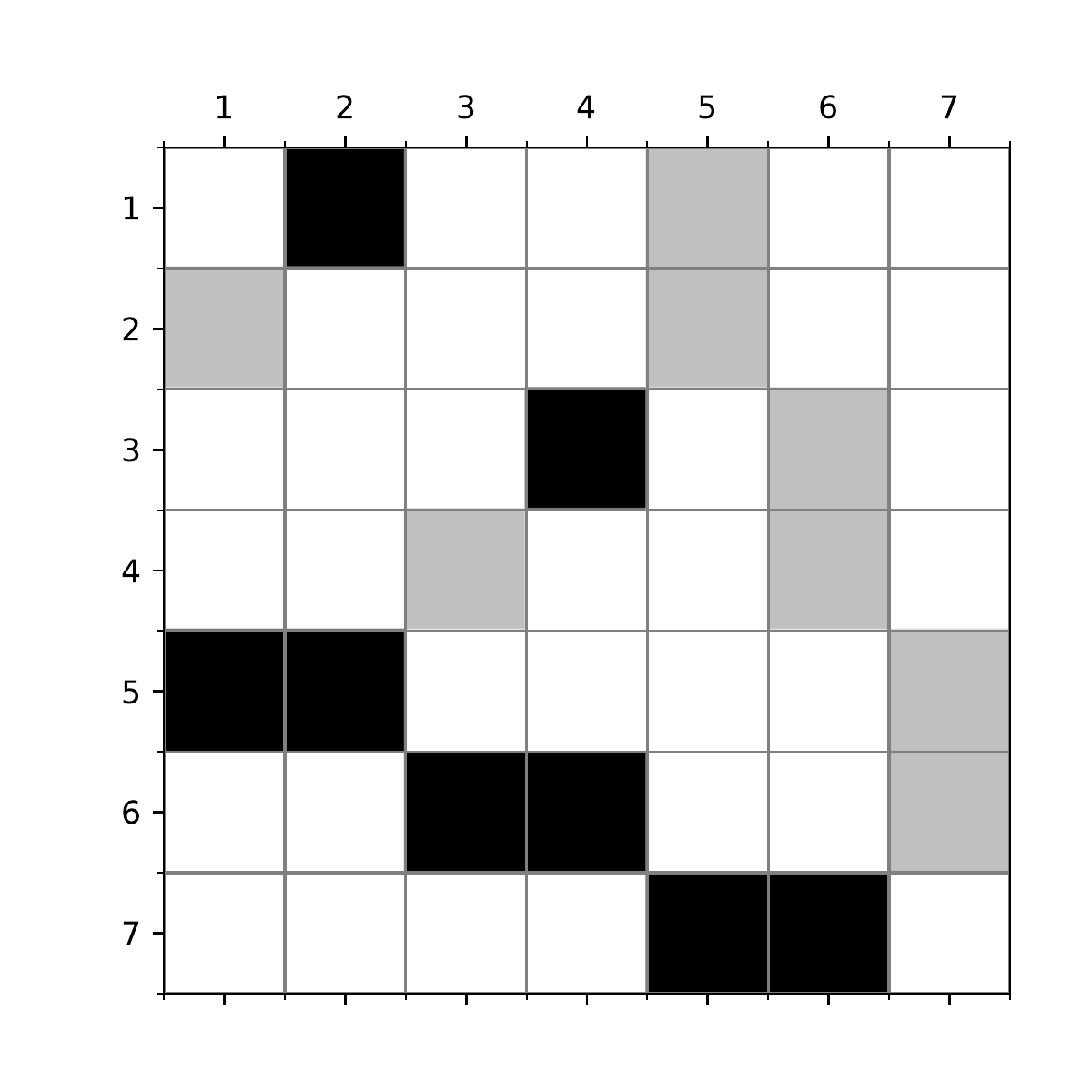}
\caption{Tree}
\end{subfigure}
\begin{subfigure}[t]{.2\textwidth}
\includegraphics[width=1.\textwidth]{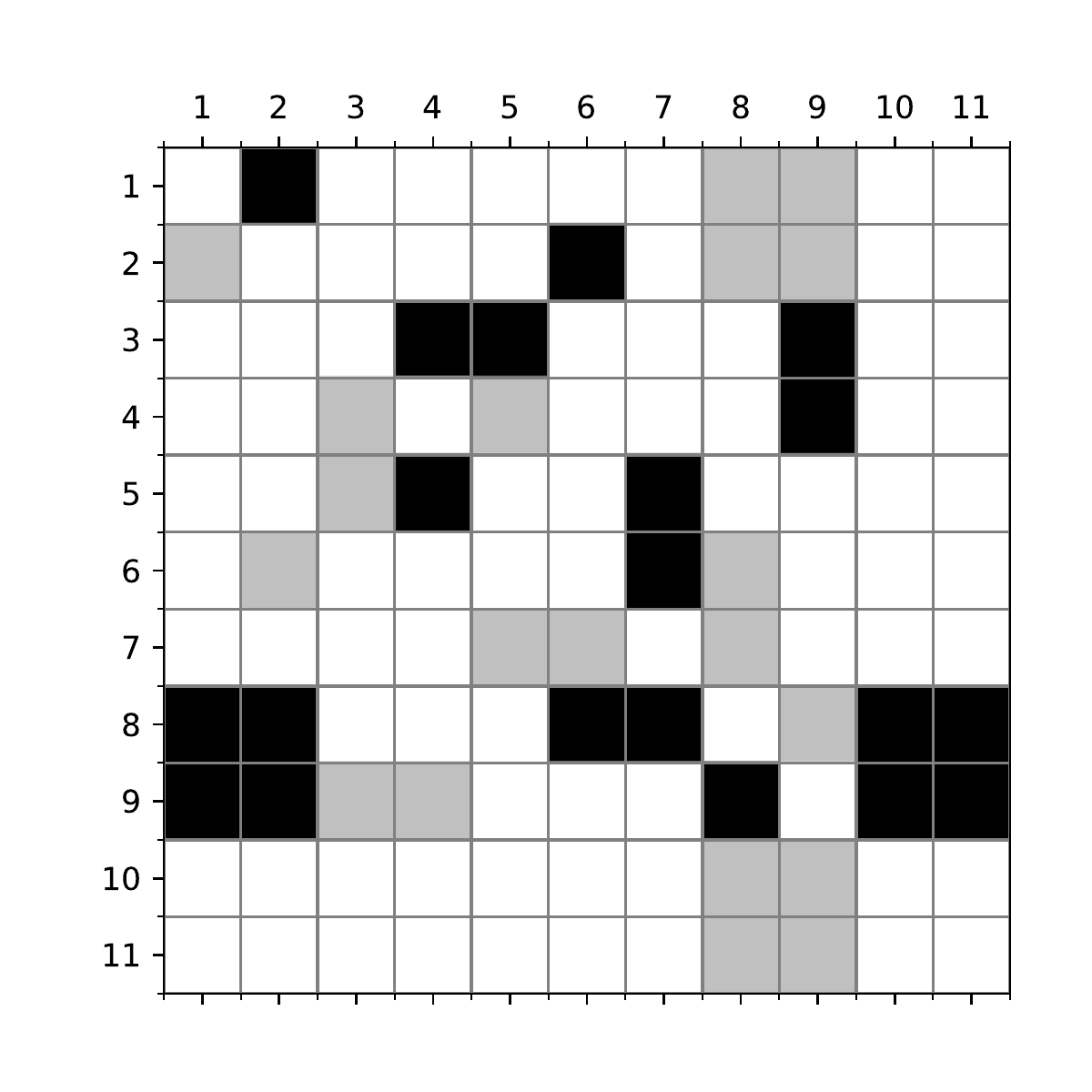}
\caption{Human Proteins}
\end{subfigure}
\end{figure*}
\paragraph{Arithmetic Circuit}
The arithmetic circuit reproduced the generative model described by \cite{wood}. It is composed of heavy tailed and conditional normal distributions, the dependencies are non-linear. We found that some of the relationships are rarely found by during topology learning, we guess that this is due to the non-linearity of the relationships which can quickly saturates and thus almost appears as constant.
\paragraph{8 pairs}
This is an artificial dataset made by us which is a concatenation of 8 2D toy problems borrowed from \cite{ffjord} implementation. These 2D variables are multi-modal and/or discontinuous. We found that learning the independence between the pairs of variables is most of the time successful even when using affine normalizers.
\paragraph{Tree}
This problem is also made on top of 2D toy problems proposed by \cite{ffjord}, in particular a sample $X = [X_1,\hdots, X_7]^T$ is generated as follows:
\begin{enumerate}
    \item The pairs variables $(X_1, X_2)$ and $(X_3, X_4)$ are respectively drawn from \textit{Circles} and \textit{8-Gaussians};
    \item $X_5 \sim \mathcal{N}(\max(X_1, X_2), 1)$;
    \item $X_6 \sim \mathcal{N}(\min(X_3, X_4), 1)$;
    \item $X_7 \sim 0.5 \mathcal{N}(\sin(X_5 + X_6), 1) + 0.5 \mathcal{N}(\cos(X_5 + X_6), 1)$.
\end{enumerate}
\paragraph{Human Proteins}
A causal protein-signaling networks derived from single-cell data. Experts have annoted 20 ground truth edges between the 11 nodes. The dataset is made of 7466 entries which we kept $5,000$ for training and $1,466$ for testing.

\subsection{Additional experiments}\label{app:top_dataset_exp}
\figref{fig:app_mono_L1} and \figref{fig:app_aff_L1} present the test log likelihood as a function of the $\ell_1$-penalization on the four datasets for monotonic and affine normalizers respectively. It can be observed that graphical conditioners perform better than autoregressive ones for certain values of regularization and when given a prescribed topology in many cases. It is interesting to observe that autoregressive architectures perform better than a prescribed topology when an affine normalizer is used. We believe this is due to the non-universality of mono-step affine normalizers which leads to different modeling trade-offs. In opposition, learning the topology improves the results in comparison to autoregressive architectures.
\begin{figure*}
\centering
\caption{Test log-likelihood as a function of $\ell_1$-penalization for monotonic normalizers. The red horizontal line is the average result when given a prescribed topology, the green horizontal line is the result with an autoregressive conditioner.}
    \label{fig:all_L1}
    \begin{subfigure}[t]{.32\textwidth}
    \includegraphics[width=1.\textwidth]{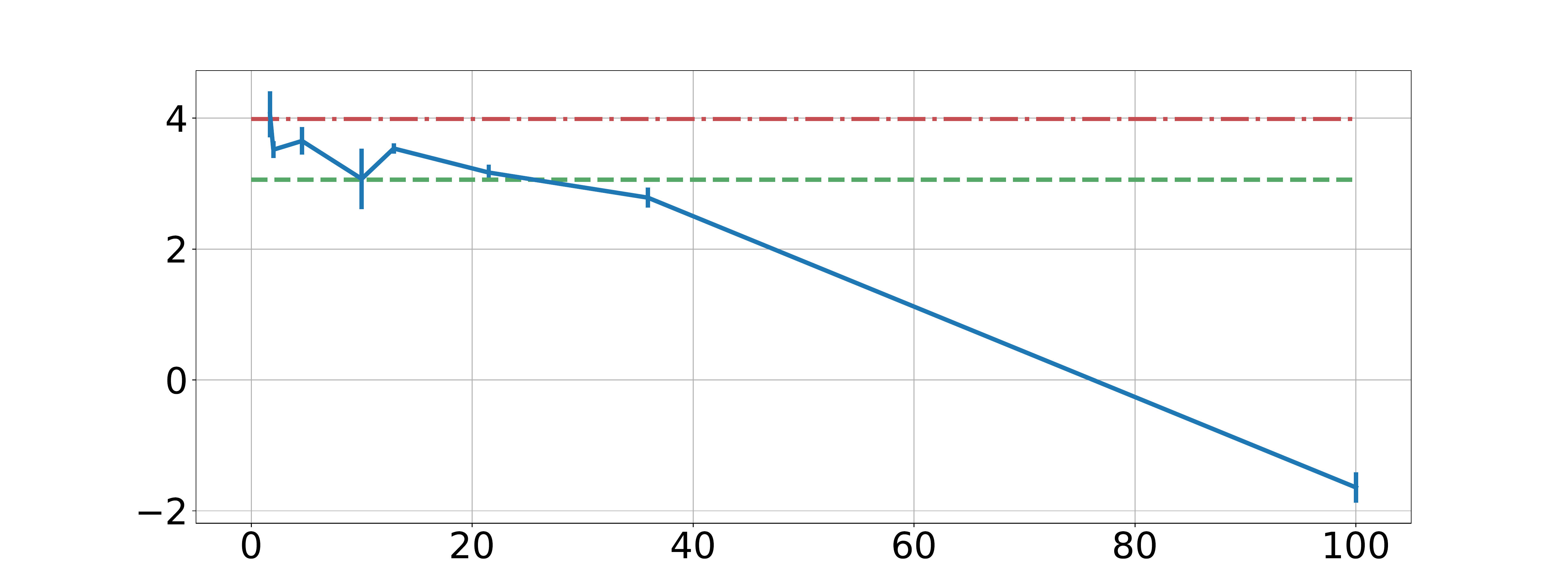}
    \caption{Arithmetic Circuit}
    \end{subfigure}
    \begin{subfigure}[t]{.32\textwidth}
    \includegraphics[width=1.\textwidth]{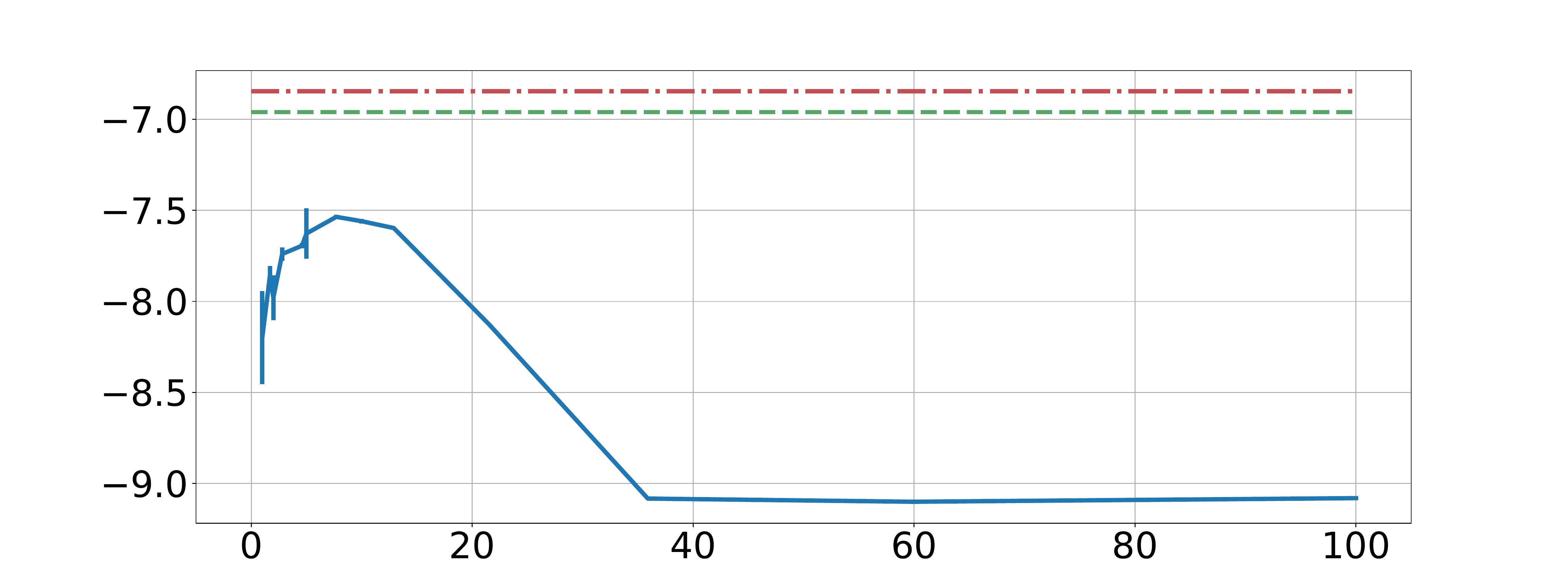}
    \caption{Tree}
    \end{subfigure}
    \begin{subfigure}[t]{.32\textwidth}
    \includegraphics[width=1.\textwidth]{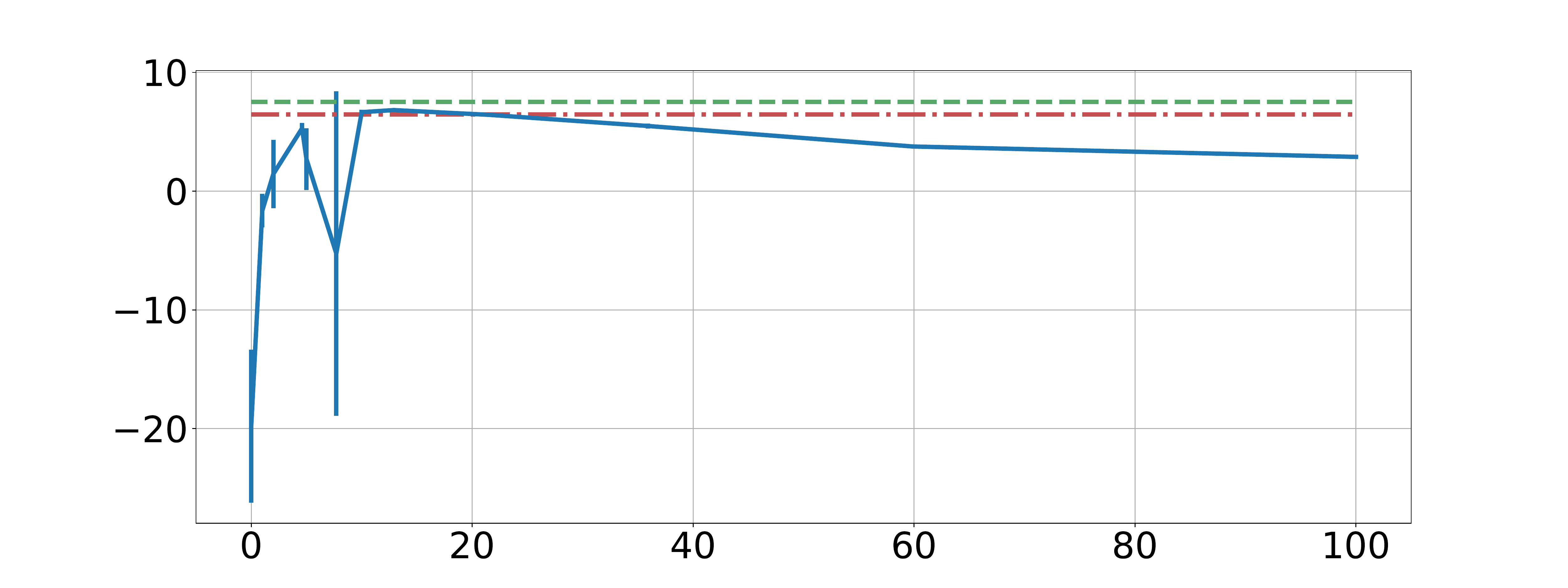}
    \caption{Human Proteins}
    \end{subfigure}
    \label{fig:app_mono_L1}
\end{figure*}
\begin{figure*}
\centering
\caption{Test log-likelihood as a function of $\ell_1$-penalization for affine normalizers. The red horizontal line is the average result when given a prescribed topology, the green horizontal line is the result with an autoregressive conditioner.}
    \label{fig:all_L1}
    \begin{subfigure}[t]{.24\textwidth}
    \includegraphics[width=1.\textwidth]{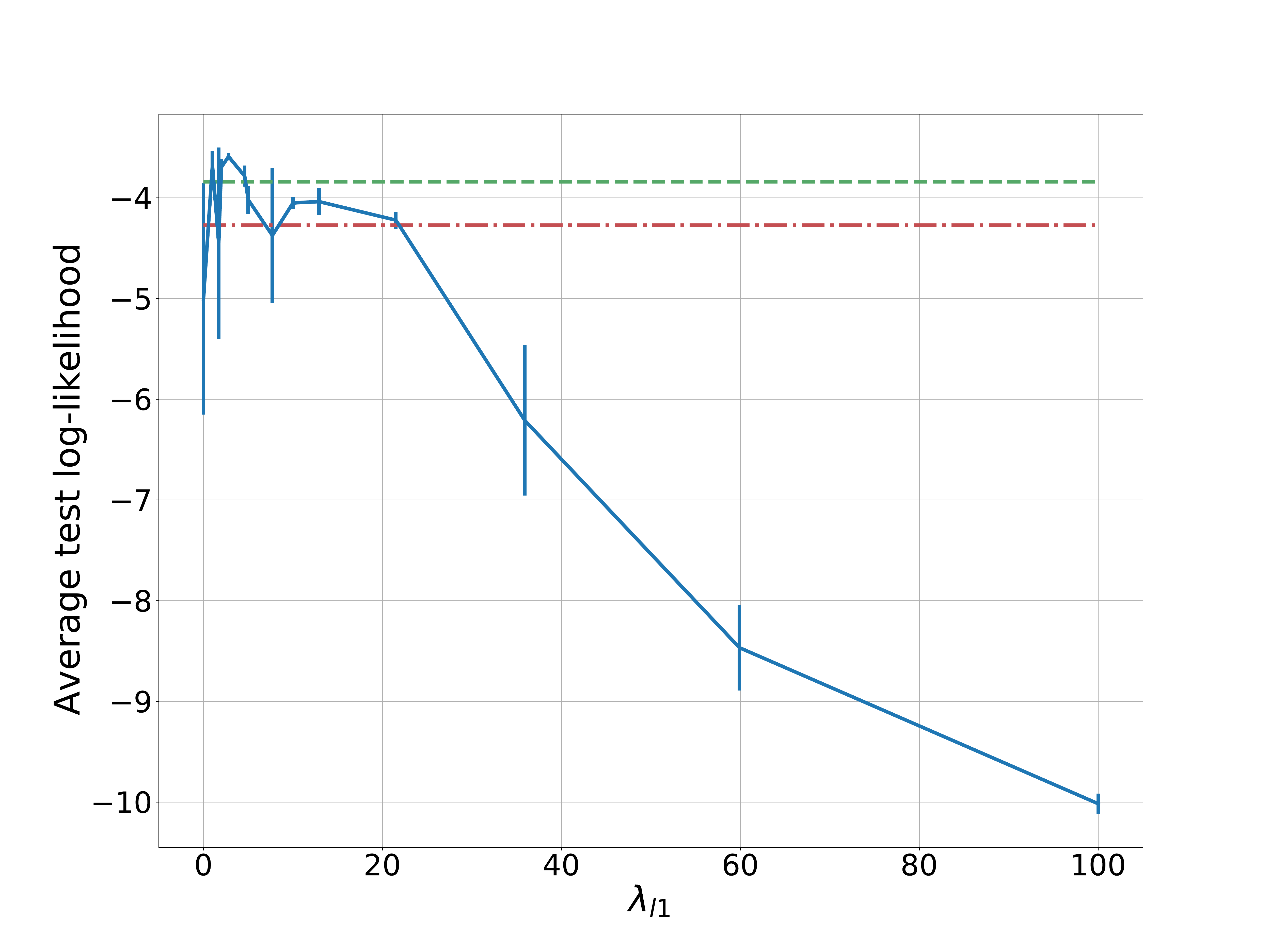}
    \caption{Arithmetic Circuit}
    \end{subfigure}
    \begin{subfigure}[t]{.24\textwidth}
    \includegraphics[width=1.\textwidth]{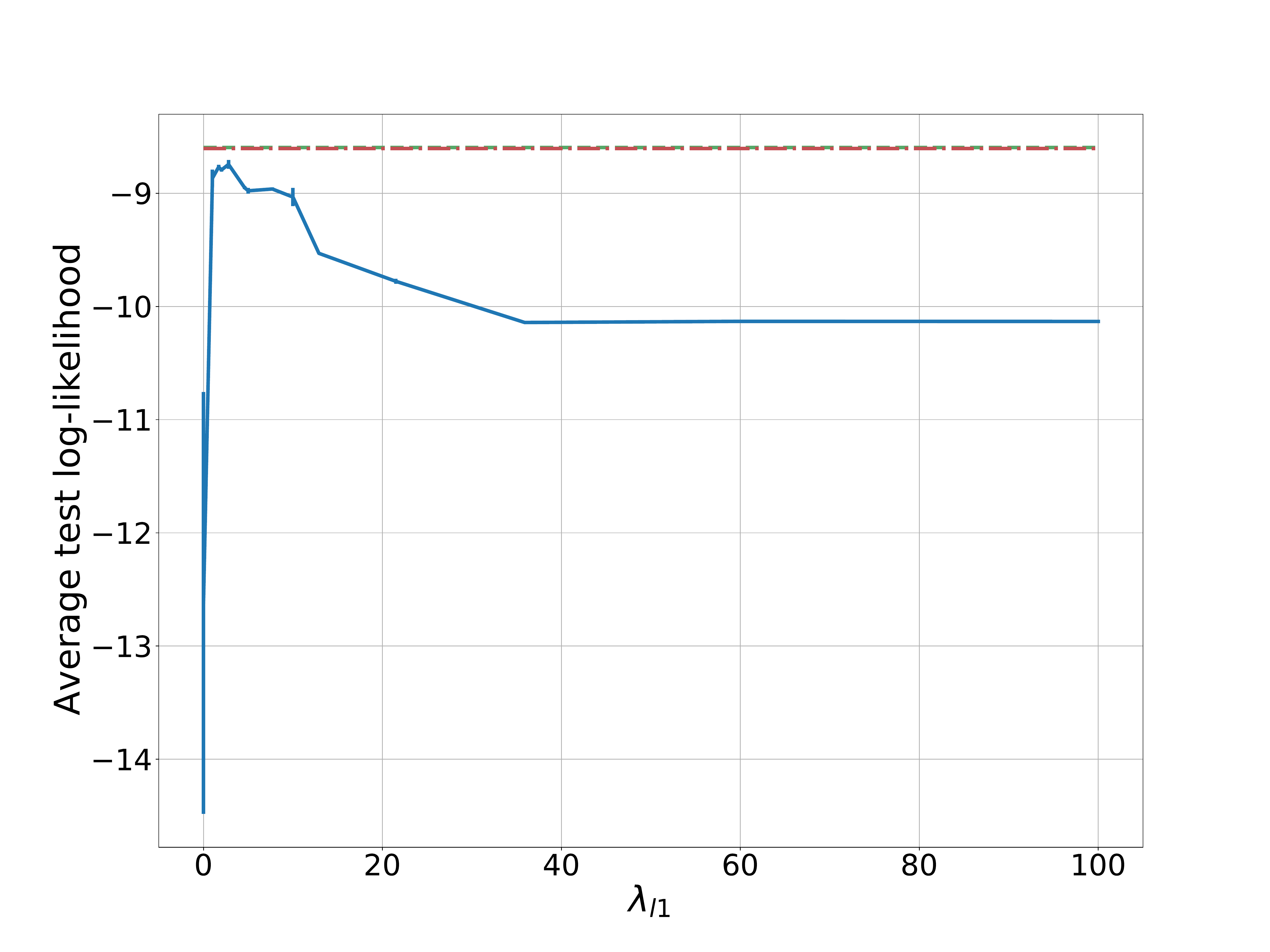}
    \caption{Tree}
    \end{subfigure}
    \begin{subfigure}[t]{.24\textwidth}
    \includegraphics[width=1.\textwidth]{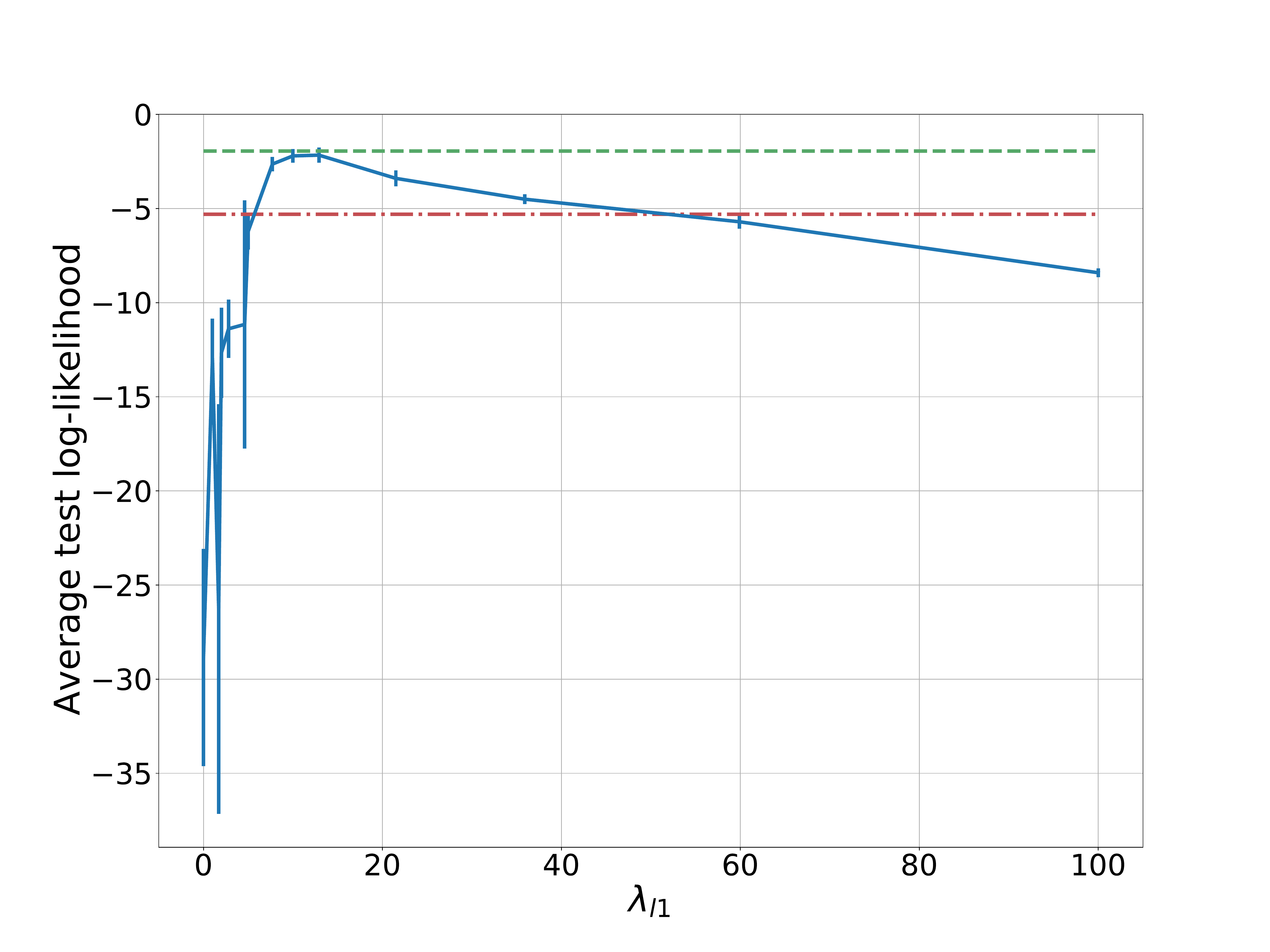}
    \caption{Human Proteins}
    \end{subfigure}
    \begin{subfigure}[t]{.24\textwidth}
    \includegraphics[width=1.\textwidth]{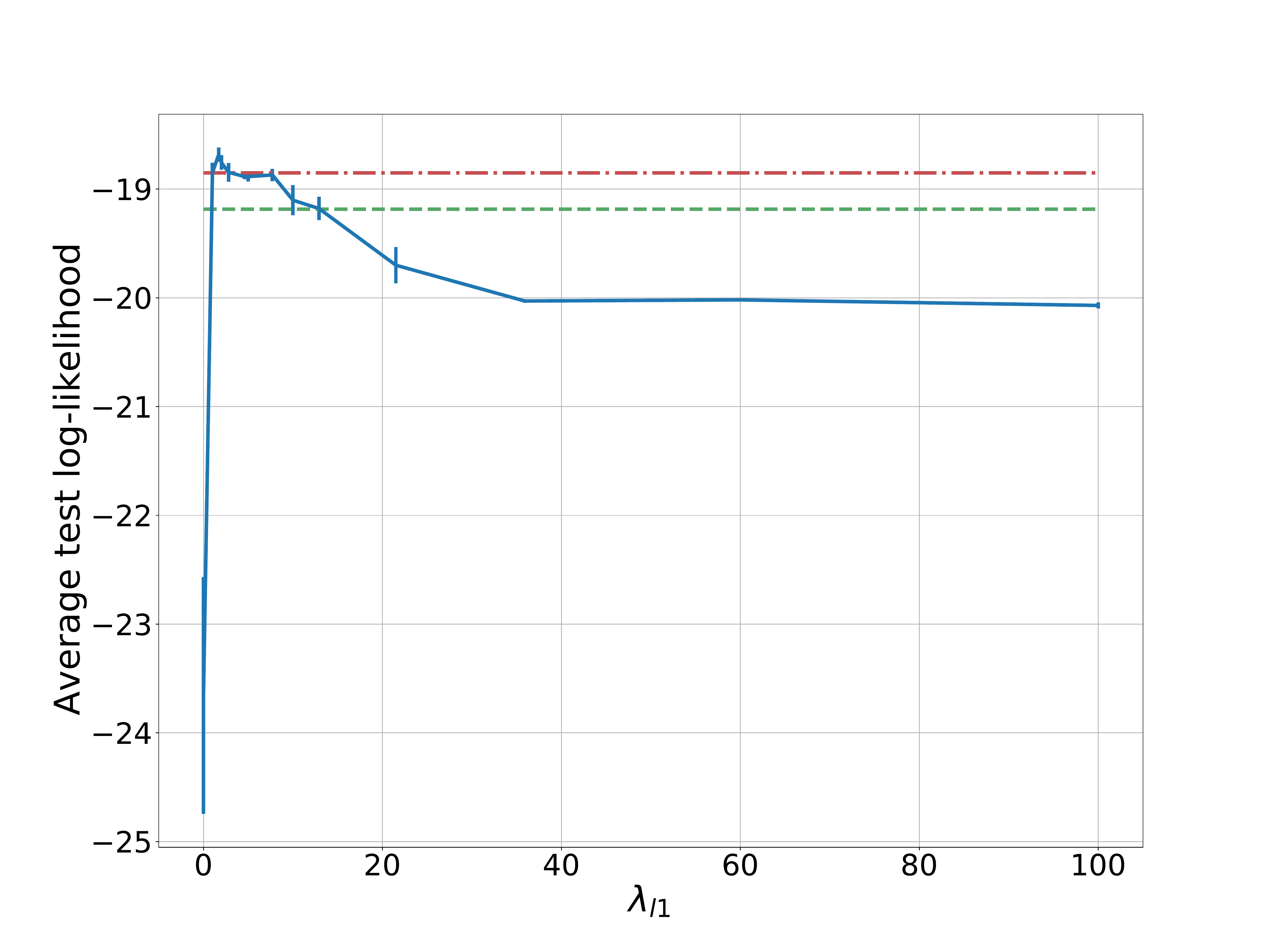}
    \caption{8 Pairs}
    \end{subfigure}
    \label{fig:app_aff_L1}
\end{figure*}
\section{Tabular density estimation - Training parameters}\label{app:archi-tabular}
\tabref{tab:train_configs} provides the hyper-parameters used to train the normalizing flows for the tabular density estimation tasks. In our experiments we parameterize the functions $\mb{h}^i$ with a unique neural network that takes a one hot encoded version of $i$ in addition to its expected input $\mb{x} \odot A_{i,:}$. The embedding net architecture corresponds to the network that computes an embedding of the conditioning variables for the coupling and DAG conditioners, for the autoregressive conditioner it corresponds to the architecture of the masked autoregressive network. The output of this network is equal to $2$ ($2 \times d$ for the autoregressive conditioner) when combined with an affine normalizer and to an hyper-parameter named \textit{embedding size} when combined with a UMNN. The number of dual steps corresponds to the number of epochs between two updates of the DAGness constraint (performed as in \cite{DAG-2}).
\begin{table}[H]
    \centering
    \tiny
    \setlength{\tabcolsep}{1pt}
    \renewcommand{\arraystretch}{1.5}
    
    \begin{tabular}{l c c c c c c c}
        \hline
        \hline
        Dataset & \tbf{POWER} & \tbf{GAS} & \tbf{HEPMASS} & \tbf{MINIBOONE} & \tbf{BSDS300}\\
        \hline
        Batch size & $2500$ & $10000$ & $100$ & $100$ & $100$ \\
        Integ. Net & $3 \times 100$ & $3 \times 200$ & $3 \times 200$ & $3 \times 40$ & $3 \times 150$ \\
        Embedd. Net & $3 \times 60$ & $3 \times 80$ & $3 \times 210$ & $3 \times 430$ & $3 \times 630$\\
        Embed. Size & $30$ & $30$ & $30$ & $30$ & $30$\\
        Learning Rate & $0.001$& $0.001$& $0.001$& $0.001$& $0.001$\\
        Weight Decay & $10^{-5}$ &  $10^{-3}$ & $10^{-4}$ & $10^{-2}$ & $10^{-4}$\\
        $\lambda_{\ell_1}$ & $0$ & $0$ & $0$ & $0$ & $0$\\
        \hline\hline
    \end{tabular}
    \vspace{1em}
    \caption{Training configurations for density estimation tasks.}
    \label{tab:train_configs}
\end{table}
In addition, in all our experiments (tabular and MNIST) the integrand networks used to model the monotonic transformations have their parameters shared and receive an additional input that one hot encodes the index of the transformed variable. The models are trained until no improvement of the average log-likelihood on the validation set is observed for 10 consecutive epochs.

\section{Density estimation of images} \label{app:MNIST}
We now demonstrate how graphical conditioners can be used to fold in domain knowledge into NFs by performing density estimation on MNIST images. The design of the graphical conditioner is adapted to images by parameterizing the functions $\mb{h}^i$ with convolutional neural networks (CNNs) whose parameters are shared for all $i \in \{1, ..., d\}$ as illustrated in \figref{fig:dag-condi-archi}. Inputs to the network $\mb{h}^i$ are masked images specified by both the adjacency matrix $A$ and the entire input image $\mb{x}$. 
Using a CNN together with the graphical conditioner allows for an inductive bias suitably designed for processing images. 
We consider single step normalizing flows whose conditioners are either coupling, autoregressive or graphical-CNN as described above, each combined with either affine or monotonic normalizers. 
The graphical conditioners that we use include an additional inductive bias that enforces a sparsity constraint on $A$ and which prevents a pixel's parents to be too distant from their descendants in the images. Formally, given a pixel located at $(i, j)$, only the pixels $(i \pm l_1, j \pm l_2),  l_1, l_2\in \{1, ..., L\}$ are allowed to be its parents. In early experiments we also tried not constraining the parents and observed slower but successful training leading to a relevant structure. 

Results reported in \tabref{tab:MNIST} show that graphical conditioners lead to the best performing affine NFs even if they are made of a single step. 
This performance gain can probably be attributed to the combination of both learning a masking scheme and processing the result with a convolutional network.
These results also show that when the capacity of the normalizers is limited, finding a meaningful factorization is very effective to improve performance. The number of edges in the equivalent BN is about two orders of magnitude smaller than for coupling and autoregressive conditioners. This sparsity is beneficial for the inversion since the evaluation of the inverse of the flow requires a number of steps equal to the depth \citep{depth-dag} of the equivalent BN. 
Indeed, we find that while obtaining density models that are as expressive, the computation complexity to generate samples is approximately divided by $\frac{5\times 784}{100} \approx 40 $ in comparison to the autoregressive flows made of 5 steps and comprising many more parameters. 

These experiments show that, in addition to being a favorable tool for introducing inductive bias into NFs, graphical conditioners open the possibility to build BNs for large datasets, unlocking the BN machinery for modern datasets and computing infrastructures. 
\begin{figure}
    \def\layersep{2.5cm}
    \centering
    \begin{tikzpicture}[shorten >=1pt,->,draw=black!50, node distance=1.25cm, scale=0.50]
    \tikzstyle{every pin edge}=[<-,shorten <=1pt]
    \tikzstyle{neuron}=[circle,fill=black!25,minimum size=7pt,inner sep=0pt]
    \tikzstyle{input neuron}=[neuron, fill=black!50];
    \tikzstyle{output neuron}=[neuron, fill=black!50];
    \tikzstyle{hidden neuron}=[neuron, fill=black!50];
    \tikzstyle{annot} = [text width=4em, text centered]
    \tikzset{edge/.style = {->,-latex}}

    \node[inner sep=0pt] (graphical) at (0,0)
    {\includegraphics[width=.55\textwidth]{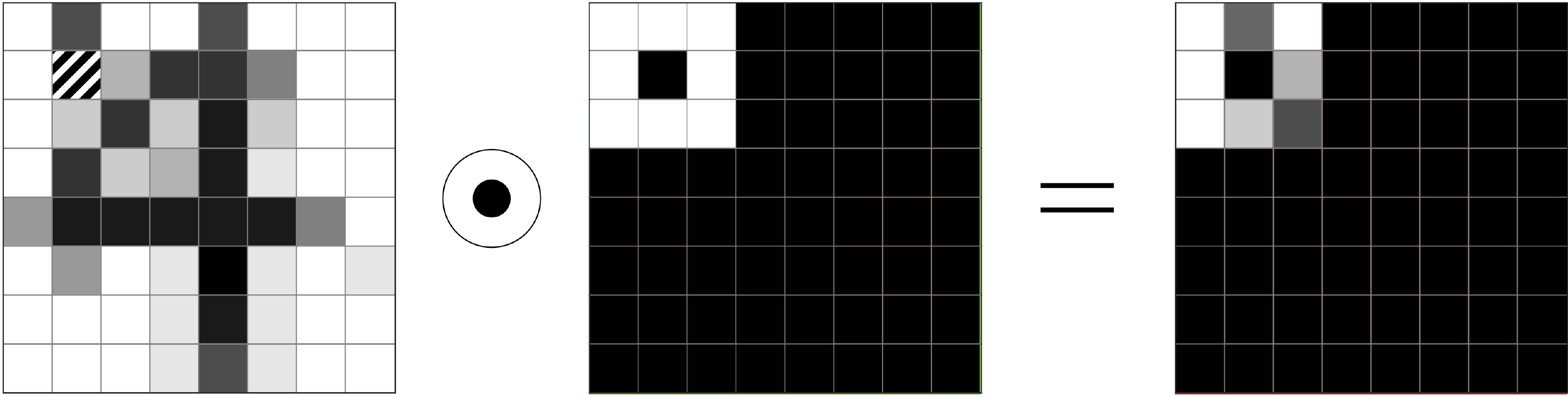}};
    \node[] (x) at (-5.5, -2.8) {$\mb{x}$};
    \node[] (A) at (0, -2.8) {$A_{i,:}$};
    \node[] (label) at (0, 3) {\small Masking operation from the adjacency matrix.};
    \node[fit={(graphical) (x) (A) (label)}] (graphical-all) {};
    \begin{scope}[shift={(12.5,0)}]
    \node[inner sep=0pt] (cnn) at (0,0){\includegraphics[width=.11\textwidth]{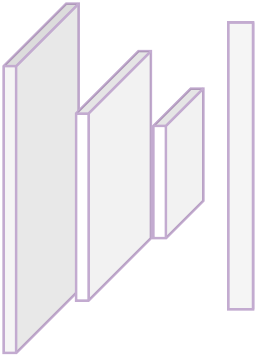}};
    \node[] (cnn-label) at (0, 3) {\small CNN};
    \node[fit={(cnn) (cnn-label)}] (cnn-all) {};
    \draw[edge, ultra thick, black]  (graphical-all) to (graphical-all-|cnn-all.west);

    
    
         \node[] at (1.8, -2.8) {$\mb{c}^i(\mb{x})$};
    
    
    \end{scope}
\end{tikzpicture}
    \caption{Illustration of how a graphical conditioner's output $\mb{c}^i(\mb{x})$ is computed for images. The sample $\mb{x}$, on the left, is an image of a $4$. The stripes denote the pixel $x_i$. The parents of $x_i$ in the learned DAG are shown as white pixels on the mask $A_{i, :}$, the other pixels are in black. The element-wise product between the image $\mb{x}$ and the mask $A_{i, :}$ is processed by a convolutional neural network that produces the embedding vector $\mb{c}^i(\mb{x})$ conditioning the pixel $x_i$.}
    \label{fig:dag-condi-archi}

\end{figure}
\begin{figure}
\begin{minipage}[h]{.48\textwidth}

\begin{table}[H]
    \centering
    \scriptsize
    \setlength{\tabcolsep}{2pt}
    \renewcommand{\arraystretch}{1.5}
    \begin{tabular}{l l | c c c c}
    \hline\hline
        & Model & Neg. LL. & Parameters & Edges & Depth \\
        \hline
        \multirow{3}{*}{(a)} 
        & G-Affine (1) & \result{1.8074786666666667}{0.006730915803629995} & $\scriptstyle 1\times 10^6$ & $\scriptstyle 5016$ & $\scriptstyle 103$ \\ 
        & G-Monotonic (1) & \result{1.1694063333333333}{0.027104726898130156} & $\scriptstyle 1\times 10^6$ & $\scriptstyle 2928$ & $\scriptstyle 125$ \\ 
        \hline
        \multirow{4}{*}{(b)} 
        & A-Affine (1) & \result{2.1196110000000004}{0.02185270459233812} & $\scriptstyle 3\times 10^6$ & $\scriptstyle 306936$ & $783$ \\ 
        & A-Monotonic (1) & \result{1.3669339999999999}{0.040867502713035936} & $\scriptstyle 3.1\times 10^6$ & $\scriptstyle 306936$ & $\scriptstyle 783$ \\ 
        & C-Affine (1) & \result{2.389779666666667}{0.02812879663658265} & $\scriptstyle 3\times 10^6$ & $\scriptstyle 153664$ & $\scriptstyle 1$ \\ 
        & C-Monotonic (1) & \result{1.6721043333333334}{0.08022989393126614} & $\scriptstyle 3.1\times 10^6$ & $\scriptstyle 153664$ & $\scriptstyle 1$ \\ 
        \hline 
        \multirow{2}{*}{(c)} 
        & A-Affine (5)  & \result{1.89}{0.01} & $\scriptstyle 6\times 10^6$  & $\scriptstyle 5\times 306936$ & $\scriptstyle 5 \times 783$\\ 
        & A-Monotonic (5) & \result{1.13}{0.02} & $\scriptstyle 6.6\times 10^6$ & $\scriptstyle 5\times 306936$  & $\scriptstyle 5 \times 783$\\ 
         \hline\hline
    \end{tabular}
    \vspace{1em}
    \caption{Results on MNIST. The negative log-likelihood is reported in bits per pixel on the test set over 3 runs on MNIST, error bars are equal to the standard deviation. The number of edges and the depth of the equivalent Bayesian network is reported. Results are divided into 3 categories: (a) The architectures introduced in this work. (b) Classical single-step architectures. (c) The best performing architectures based on multi-steps autoregressive flows.} \label{tab:MNIST}
\end{table}
\end{minipage}
\begin{minipage}[h]{.01\textwidth}
\hspace{1.\textwidth}
\end{minipage}
\begin{minipage}[h]{.49\textwidth}

\begin{figure}[H]
    \centering
    \includegraphics[width=\textwidth]{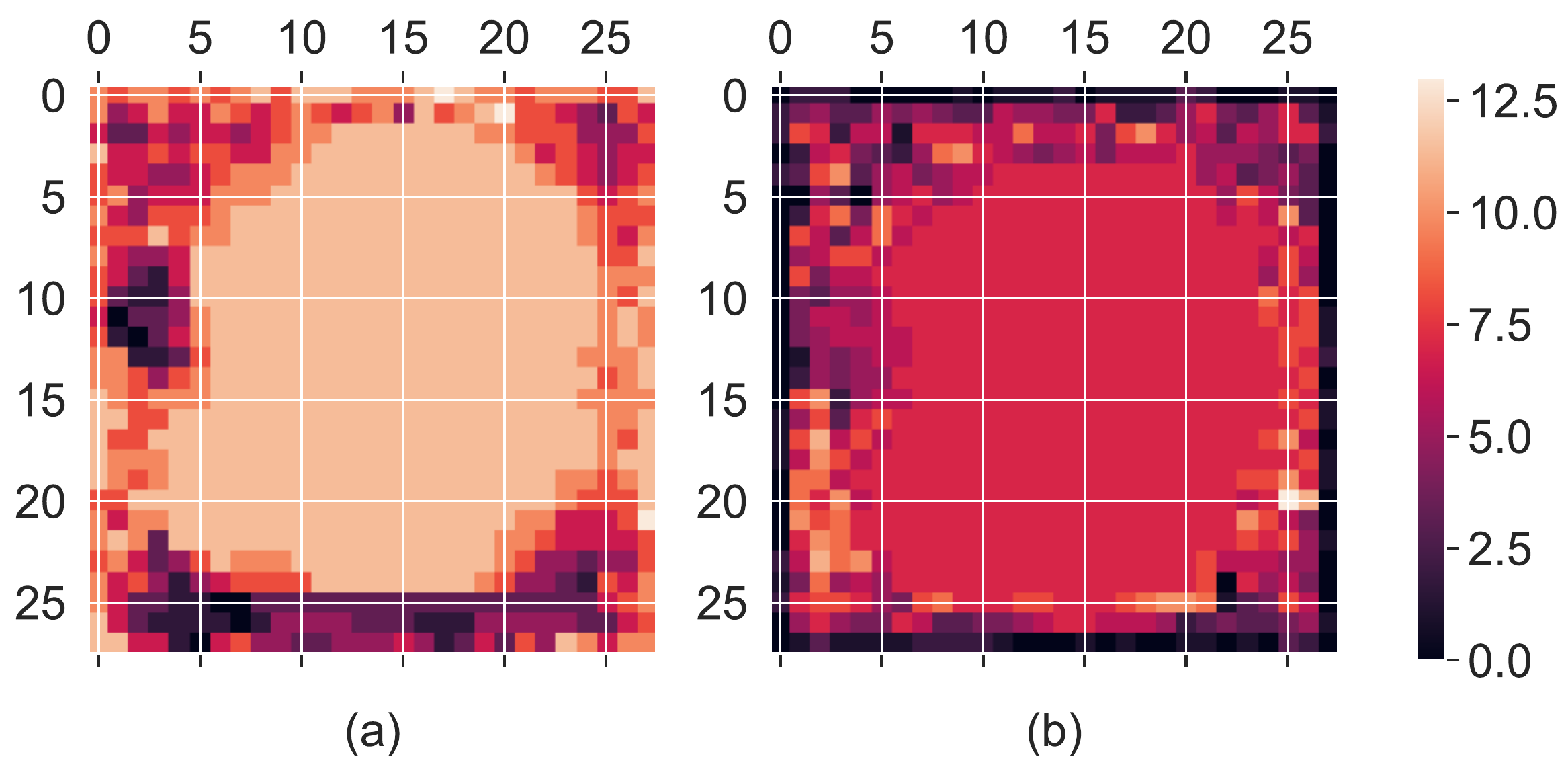}
    \caption{The in (a) and out (b) degrees of the nodes in the equivalent BN learned in the MNIST experiments.} \label{fig:in-out-degrees}
\end{figure}


\end{minipage}
\end{figure}

\section{MNIST density estimation - Training parameters}
For all experiments the batch size was $100$, the learning rate $10^{-3}$, the weight decay $10^{-5}$. For the graphical conditioners the number of epochs between two coefficient updates was chosen to $10$, the greater this number the better were the performance however the longer is the optimization. The CNN is made of 2 layers of 16 convolutions with $3\times 3$ kernels followed by an MLP with two hidden layers of size $2304$ and $128$. The neural network used for the Coupling and the autoregressive conditioner are neural networks with $3 \times 1024$ hidden layers. For all experiments with a monotonic normalizer the size of the embedding was chosen to $30$ and the integral net was made of 3 hidden layers of size $50$.
The models are trained until no improvements of the average log-likelihood on the validation set is observed for 10 consecutive epochs.

\end{document}